\documentclass[dvips]{article}

\RequirePackage[OT1]{fontenc}
\RequirePackage{amsthm,amsmath}
\RequirePackage[numbers]{natbib}
\RequirePackage[colorlinks,citecolor=blue,urlcolor=blue]{hyperref}
\RequirePackage{hypernat}

\newtheorem{thm}{Theorem}[section]

\usepackage{graphicx}
\usepackage{amsfonts}
\usepackage{amssymb}
\usepackage{mathrsfs}

\newcommand{\eps}{\varepsilon}

\newcommand{\R}{{\mathbb{R}}}

\newcommand{\cH}{\mathcal{H}}
\newcommand{\cX}{\mathcal{X}}
\newcommand{\cY}{\mathcal{Y}}

\newcommand{\Hx}{{\cH_\cX}}
\newcommand{\Hy}{{\cH_\cY}}

\newcommand{\hC}{\widehat{C}^{(n)}}

\newcommand{\hmu}{\widehat{\mu}}

\newcommand{\ty}{\tilde{y}}

\newcommand{\hm}{\widehat{m}^{(n)}}
\newcommand{\hml}{\widehat{m}^{(\ell)}}
\newcommand{\la}{\langle}
\newcommand{\ra}{\rangle}

\newcommand{\eq}[1]{Eq.~(\ref{#1})}
\newtheorem{lma}[thm]{Lemma}
\newtheorem{prop}[thm]{Proposition}
\title{Kernel Bayes' Rule}%\protect\thanksref{T1}}

\author{Kenji Fukumizu\thanks{The Institute of Statistical Mathematics, fukumizu@ism.ac.jp}, Le Song\thanks{Carnegie Mellon University, lesong@cs.cmu.edu}, Arthur Gretton\thanks{University College London, and Max Planck Institute for Biological Cybernetics, arthur.gretton@googlemail.com}}

\date{\today}

\begin{document}

\maketitle

\begin{abstract}
A nonparametric kernel-based method for realizing Bayes' rule is proposed, based on representations of probabilities in reproducing kernel Hilbert spaces. Probabilities are uniquely characterized by the mean of the canonical map to the RKHS. The prior and conditional probabilities are expressed in terms of RKHS functions of an empirical sample: no explicit parametric model is needed for these quantities. The posterior is likewise an RKHS mean of a weighted sample. The estimator for the expectation of a function of the posterior is derived, and rates of consistency are shown.   Some representative applications of the kernel Bayes' rule  are presented, including Baysian computation without likelihood and filtering with a nonparametric state-space model.

\end{abstract}

\section{Introduction}

Kernel methods have long provided  powerful tools for generalizing linear statistical approaches to nonlinear settings, through an embedding of the sample to a high dimensional feature space, namely a reproducing kernel Hilbert space (RKHS)  \cite{SchoelkopfSmola_book,Hofmann_etal_2008_kernel_AS}.  Examples include support vector machines, kernel PCA, and kernel CCA, among others.
In these cases, data are mapped via a canonical feature map to a reproducing kernel Hilbert space (of high or even infinite dimension), in which the linear operations that define the algorithms are implemented. The inner product between  feature mappings need never be computed explicitly, but is given by a positive definite kernel function unique to the RKHS: this permits efficient computation without the need to deal explicitly with the feature representation.

The mappings of individual points to a feature space may be generalized to mappings of probability measures \citep[e.g.][Chapter 4]{Berlinet_RKHS}. We call such mappings the {\em kernel means} of the underlying random variables. With an appropriate choice of positive definite kernel,
the kernel mean on the RKHS uniquely determines the distribution of the variable
 \cite{Fukumizu04_jmlr,Fukumizu_etal09_KDR,Sriperumbudur_etal2010JMLR}, and
 statistical inference problems on distributions can be solved via operations on the kernel means.
Applications of this approach include  homogeneity testing
\cite{Gretton_etal07,Gretton_etal_nips2009}, where the empirical means on the RKHS are compared directly,
and independence testing \cite{Gretton_etal_nips07,Gretton_etal_AOAS2009}, where the mean of the joint distribution on the feature space is compared with that of the product of the marginals.
Representations of conditional dependence may also be defined in RKHS, and have been used in
conditional independence tests
\cite{Fukumizu_etal_nips07}.
%Recently, \cite{Song_etal_ICML2009} used conditional distribution embeddings to perform inference for a nonlinear state-space model.

In this paper, we propose a novel, nonparametric approach to Bayesian inference, making use of kernel means of probabilities. In applying Bayes' rule, we compute the posterior probability of $x$ in $\cX$ given observation $y$ in $\cY$;
\begin{equation}\label{eq:BayesRule}
    q(x|y) = \frac{p(y|x)\pi(x)}{q_\cY(y)},
\end{equation}
where $\pi(x)$ and $p(y|x)$ are the density functions of the prior and the likelihood of $y$ given $x$, respectively, with respective base measures $\nu_\cX$ and $\nu_\cY$, and the normalization factor $q_\cY$(y) is given by
\begin{equation}\label{eq:BayesForward}
    q_\cY(y) = \int p(y|x)\pi(x)d\nu_\cX(x).
\end{equation}
Our main result is a nonparametric estimate of the kernel mean posterior, given kernel mean representations of the prior and likelihood.

A valuable property of the kernel Bayes' rule is that the kernel posterior mean  is estimated nonparametrically from data; specifically, the  prior and the likelihood are represented in the form of samples from the prior and the joint probability that gives the likelihood, respectively.  This confers an important benefit: we can still perform Bayesian inference by making sufficient observations on the system, even in the absence of a specific parametric model of the relation between variables. More generally, if we can sample from the model, we do not require explicit density functions for inference.
 Such situations are typically seen when the prior or likelihood is given by a random process: Approximate Bayesian Computation \citep{Tavare_etal_1997ABC,Marjoram_etal_2003PNAS,Sisson_etal2007} is widely applied in population genetics, where the likelihood is given by a branching process, and nonparametric Bayesian inference \citep{MullerQuintana2004} often uses a process prior with sampling methods. Alternatively, a parametric model may be known, however it might be of sufficient complexity to require  Markov chain Monte Carlo or sequential Monte Carlo for inference.  The present kernel approach provides an alternative strategy for Bayesian inference in these settings.  We demonstrate rates of consistency for our posterior kernel mean estimate, and for the expectation of functions computed using this estimate.

An alternative to the kernel mean representation would be to use nonparametric density estimates for the posterior. Classical approaches include kernel density estimation (KDE) or distribution estimation on a finite partition of the domain.  These methods are known to perform poorly on high dimensional data, however.
%Another method is to use direct estimates of the density ratio \cite{Sugiyama_etal_AISTAT2010}, which could be used in estimating the conditional probability density function as a step towards Bayesian inference.
By contrast, the proposed kernel mean representation is defined as an integral or moment of the distribution, taking the form of a function in an RKHS.  Thus, it is more akin to the characteristic function approach (see e.g. \cite{KankainenUshakov1998}) to representing probabilities.  A well conditioned empirical estimate of the characteristic function can be difficult to obtain,  especially for conditional probabilities.  By contrast, the kernel mean has a straightforward empirical estimate, and conditioning and marginalization can be implemented easily, at a reasonable computational cost.

The proposed method of realizing Bayes' rule is an extension of the approach used in \cite{Song_etal_ICML2009} for state-space models. In this earlier work, a heuristic approximation was used, where the kernel mean of the new hidden state was estimated by adding  kernel mean estimates from the previous hidden state and the observation. Another relevant work is the belief propagation approach in \cite{Song_etal_AISTATS2010,SonGreBicLowGue11}, which covers the simpler case of a uniform prior.

This paper is organized as follows.  We begin in Section \ref{sec:kernel} with a review of RKHS terminology and of kernel mean embeddings. In Section \ref{sec:KBR}, we derive an expression for Bayes' rule in terms of kernel means, and provide consistency guarantees.  We apply the kernel Bayes' rule in Section \ref{sec:KBRmethods}  to various inference problems, with numerical results and comparisons with existing methods  in Section \ref{sec:experiments}.   Our proofs are contained in Section \ref{sec:proof} (including proofs of the consistency results of Section \ref{sec:KBR}).
%Section \ref{sec:conclusion} concludes the paper.

\section{Preliminaries: positive definite kernel and probabilities}
\label{sec:kernel}

Throughout this paper, all Hilbert spaces are assumed to be separable. For an operator $A$ on a Hilbert space, the range is denoted by $\mathcal{R}(A)$.  The linear hull of a subset $S$ in a vector space is denoted by ${\rm Span}S$.

We begin with a review of positive definite kernels, and of statistics on the associated reproducing kernel Hilbert spaces \citep{Aronszajn50,Berlinet_RKHS,Fukumizu04_jmlr,Fukumizu_etal09_KDR}.
Given a set $\Omega$, a ($\R$-valued) positive definite kernel $k$ on $\Omega$ is a symmetric kernel $k:\Omega\times\Omega\to\R$ such that $\sum_{i,j=1}^n c_i c_j k(x_i,x_j)\geq 0$ for arbitrary number of points $x_1,\dots,x_n$ in $\Omega$ and real numbers $c_1,\ldots,c_n$.  The matrix $(k(x_i,x_j))_{i,j=1}^n$ is called a Gram matrix.  It is known by the Moore-Aronszajn theorem \citep{Aronszajn50} that a positive definite kernel on $\Omega$ uniquely defines a Hilbert space $\cH$ consisting of functions on $\Omega$ such that (i) $k(\cdot,x)\in\cH$ for any $x\in\Omega$, (ii) ${\rm Span}\{ k(\cdot,x)\mid x\in \Omega\}$ is dense in $\cH$, and (iii) $\la f,k(\cdot,x)\ra = f(x)$ for any $x\in \Omega$ and $f\in\cH$ (the reproducing property), where $\la\cdot,\cdot\ra$ is the inner product of $\cH$.
The Hilbert space $\cH$ is called the {\it reproducing kernel Hilbert space} (RKHS) associated with $k$, since the function $k_x = k(\;,x)$ serves as the reproducing kernel $\la f, k_x\ra=f(x)$ for $f\in\cH$.

A positive definite kernel on $\Omega$ is said to be {\em bounded} if there is $M>0$ such that $k(x,x)\leq M$ for any $x\in \Omega$.

Let $(\cX,\mathcal{B}_\cX)$ be a measurable space, $X$ be a random variable taking values in $\cX$ with distribution $P_X$, and $k$ be a measurable positive definite
kernel on $\cX$ such that $E[\sqrt{k(X,X)}]<\infty$.  The associated RKHS is denoted by $\cH$.
The {\em kernel mean} $m_X^k$ (also written $m_{P_X}^k$)
of $X$ on the RKHS $\cH$ is defined by the mean of the $\cH$-valued random variable $k(\cdot,X)$.  The existence of the kernel mean is guaranteed by $E[\|k(\cdot,X)\|] =E[\sqrt{k(X,X)}]<\infty$.  We usually write $m_X$ for $m_X^k$ for simplicity, where there is no ambiguity.
  By the reproducing property, the kernel mean satisfies the relation
\begin{equation}\label{eq:reproducing_mean}
    \la f,m_X\ra = E[f(X)]
\end{equation}
for any $f\in\cH$.  Plugging $f=k(\cdot,u)$ into this relation derives
\begin{equation}\label{eq:mean_integ}
    m_X(u) = E[k(u, X)] = \int k(u,\tilde{x})dP_X(\tilde{x}),
\end{equation}
which shows the explicit functional form.
The kernel mean $m_X$ is also denoted by $m_{P_X}$, as it depends only on the distribution $P_X$ with $k$ fixed.

Let $(\cX,\mathcal{B}_\cX)$ and
$(\cY,\mathcal{B}_\cY)$ be measurable spaces, $(X,Y)$ be
a random variable on $\cX\times\cY$ with distribution $P$, and $k_\cX$ and $k_\cY$ be measurable positive definite kernels with respective RKHS $\Hx$ and $\Hy$ such that $E[k_\cX(X,X)]<\infty$ and $E[k_\cY(Y,Y)]<\infty$.
The (uncentered) {\em covariance operator} $C_{YX}:
\Hx\to\Hy$ is defined as the linear operator that satisfies
\[
    \la g, C_{YX} f\ra_\Hy = E[f(X)g(Y)]
\]
for all $f\in\Hx,g\in\Hy$.  This operator $C_{YX}$ can be identified with $m_{(YX)}$ in the product space $\Hy\otimes\Hx$, which is given by the product kernel $k_\cY k_\cX$ on $\cY\times\cX$ \citep{Aronszajn50}, by the standard identification between the linear maps and the tensor product.
We also define $C_{XX}$ for the operator on $\Hx$ that satisfies $\la f_2, C_{XX}f_1\ra = E[f_2(X)f_1(X)]$ for any $f_1,f_2\in\Hx$.  Similarly to \eq{eq:mean_integ}, the explicit integral expressions for $C_{YX}$ and $C_{XX}$ are given by
\begin{equation}\label{eq:cov_op_integ}
(C_{YX}f )(y) = \int k_\cY(y,\tilde{y})f(\tilde{x})dP(\tilde{x},\tilde{y}), \quad
(C_{XX}f )(x) = \int k_\cX(x,\tilde{x})f(\tilde{x})dP_X(\tilde{x}),
\end{equation}
respectively.

An important notion in statistical inference with positive definite kernels is the characteristic property.  A bounded measurable positive definite kernel $k$ on a measurable space $(\Omega, \mathcal{B})$ is called {\em characteristic} if the mapping from a probability $Q$ on $(\Omega, \mathcal{B})$ to the kernel mean $m_Q^k\in \cH$ is injective \citep{Fukumizu_etal09_KDR,Sriperumbudur_etal2010JMLR}.  This is equivalent to assuming that $E_{X\sim P}[k(\cdot,X)] = E_{X'\sim Q}[k(\cdot,X')]$ implies $P=Q$:  probabilities are uniquely determined by their kernel means on the associated RKHS.  With this property, problems of statistical inference can be cast as inference on the kernel means.
A popular example of a characteristic kernel defined on Euclidean space is the Gaussian RBF kernel $k(x,y) = \exp(-\|x-y\|^2/(2\sigma^2))$.  It is known that a bounded measurable positive definite kernel on a measurable space $(\Omega,\mathcal{B})$ with corresponding RKHS $\cH$ is characteristic if and only if $\cH + \R$ is dense in $L^2(P)$ for arbitrary probability $P$ on $(\Omega,\mathcal{B})$, where $\cH+\R$ is the direct sum of two RKHSs $\cH$ and $\R$ \cite{Aronszajn50}.  This implies that the RKHS defined by a characteristic kernel is rich enough to be dense in $L^2$ space up to the constant functions.  Other useful conditions for a kernel to be characteristic can be found in \cite{Sriperumbudur_etal2010JMLR,FukSriGreSch09,Sriperumbudur_etal2011JMLR}.

Throughout this paper, when positive definite kernels on a measurable space are discussed, the following assumption is made:
\begin{description}
\item[(K)] Positive definite kernels are bounded and measurable.
\end{description}
Under this assumption, the mean and covariance always exist with arbitrary probabilities.

Given i.i.d.~sample $(X_1,Y_1),\ldots,(X_n,Y_n)$ with law $P$, the empirical estimator of the kernel mean and covariance operator are given straightforwardly by
\[
    \hm_X = \frac{1}{n}\sum_{i=1}^n k_\cX(\cdot,X_i),\qquad
    \hC_{YX} = \frac{1}{n} \sum_{i=1}^n k_\cY(\cdot,Y_i)\otimes k_\cX(\cdot,X_i),
\]
where $\hC_{YX}$ is written in  tensor form.  It is known that these estimators are $\sqrt{n}$-consistent in appropriate norms, and $\sqrt{n}(\hm_X - m_X)$ converges to a Gaussian process on $\Hx$ \cite[][Sec. 9.1]{Berlinet_RKHS}.  While we may use non-i.i.d.~samples for numerical examples in Section \ref{sec:experiments}, in our theoretical analysis we always assume i.i.d.~samples for simplicity.

\section{Kernel expression of Bayes' rule}
\label{sec:KBR}

\subsection{Kernel Bayes' rule}

Let $(\cX,\mathcal{B}_\cX)$ and
$(\cY,\mathcal{B}_\cY)$ be measurable spaces, $(X,Y)$ be
a random variable on $\cX\times\cY$ with distribution $P$, and $k_\cX$ and $k_\cY$ be positive definite kernels on $\cX$ and $\cY$, respectively, with respective RKHS $\Hx$ and $\Hy$.
Let $\Pi$ be a probability on $(\cX,\mathcal{B}_\cX)$, which
serves as a {\em prior} distribution.  For each $x\in\cX$, define a probability $P_{Y|x}$ on $(\cY,\mathcal{B}_\cY)$ by $P_{Y|x}(B)=E[I_B(Y)|X=x]$, where $I_B$ is the index function of a measurable set $B\in\mathcal{B}_\cY$.  The prior $\Pi$ and the family $\{P_{Y|x}\mid x\in\cX\}$ defines the joint distribution $Q$ on $\cX\times\cY$ by
\[
    Q(A\times B)=\int_A P_{Y|x}(B)d\Pi(x)
\]
for any $A\in\mathcal{B}_\cX$ and $B\in\mathcal{B}_\cY$, and its marginal distribution $Q_\cY$ by $Q_\cY(B)=Q(\cX\times B)$.  Throughout this paper, it is assumed that $P_{Y|x}$ and $Q$ are well-defined under some regularity conditions.  Let $(Z,W)$ be a random variable on $\cX\times\cY$ with distribution $Q$.  It is also assumed that the sigma algebra generated by $W$ includes every point $\{y\}$ ($y\in\cY$).  For $y\in\cY$, the {\em posterior} probability given $y$ is defined by the conditional probability
\begin{equation}\label{eq:BayesRule_general}
Q_{\cX|y}(A) = E[I_A(Z)|W=y] \qquad (A\in\mathcal{B}_\cX).
\end{equation}
If the probability distributions have density functions with respect to measures $\nu_\cX$ on $\cX$ and $\nu_\cY$ on $\cY$, namely, if the p.d.f.~of $P$ and $\Pi$ are given by $p(x,y)$ and $\pi(x)$, respectively, \eq{eq:BayesRule_general} is reduced to the well known form \eq{eq:BayesRule}.

The goal of this subsection is to derive an estimator of the kernel mean of posterior $m_{Q_\cX|y}$.  The following theorem is fundamental to discuss conditional probabilities with positive definite kernels.
\begin{thm}[\cite{Fukumizu04_jmlr}]\label{thm:cond_mean}
If $E[g(Y)|X=\cdot]\in\Hx$ holds for $g\in\Hy$, then
\[
    C_{XX} E[g(Y)|X=\cdot]=C_{XY} g.
\]
\end{thm}
If $C_{XX}$ is injective, i.e., if the function
$f\in\Hx$ with $C_{XX} f = C_{XY} g$ is unique, the above relation can be expressed as
\begin{equation}\label{eq:cond_basic}
    E[g(Y)|X=\cdot]={C_{XX}}^{-1}C_{XY} g.
\end{equation}
Noting $\la C_{XX}f, f\ra=E[f(X)^2]$, it is easy to see that $C_{XX}$ is injective, if $\cX$ is a topological space, $k_\cX$ is a continuous kernel, and ${\rm Supp}(P_X) = \cX$, where ${\rm Supp}(P_X)$ is the support of $P_X$.

From Theorem \ref{thm:cond_mean},
we have the following result, which expresses the kernel mean of $Q_\cY$.
\begin{thm}[\cite{Song_etal_ICML2009}, Eq. 6]
\label{thm:cond_prob_op}
Let $m_{\Pi}$ and
$m_{Q_\cY}$ be the kernel means of $\Pi$ in $\Hx$ and $Q_\cY$ in $\Hy$, respectively.  If $C_{XX}$ is injective, $m_\Pi\in \mathcal{R}(C_{XX})$, and $E[g(Y)|X=\cdot]\in\Hx$ for
any $g\in \Hy$, then
\begin{equation}\label{eq:cond_prob_op}
   m_{Q_\cY}= C_{YX} {C_{XX}}^{-1} m_{\Pi}.
\end{equation}
\end{thm}
\begin{proof}
Take $f\in\Hx$ such that $f=C_{XX}^{-1}m_{\Pi}$.  For any $g\in \Hy$, $\la C_{YX}f, g\ra = \la f, C_{XY}g\ra = \la f,
C_{XX}E[g(Y)|X=\cdot]\ra = \la C_{XX}f, E[g(Y)|X=\cdot]\ra  = \la
m_\Pi, E[g(Y)|X=\cdot]\ra =  \la m_{Q_{\cY}}, g\ra$, which implies
 $C_{YX}f=m_{Q_{\cY}}$.
\end{proof}
As discussed in \cite{Song_etal_ICML2009}, the operator
$C_{YX}C_{XX}^{-1}$ can be regarded as the kernel expression of
the conditional probability $P_{Y|x}$ or $p(y|x)$.

Note, however, that the assumption $E[g(Y)|X=\cdot]\in \Hx$ may not hold in general; we can easily give counterexamples in the case of Gaussian kernels\footnote{Suppose that $\Hx$ and $\Hy$ are given by Gaussian kernel, and that $X$ and $Y$ are independent.  Then, $E[g(Y)|X=x]$ is a constant function of $x$, which is known not to be included in a RKHS given by a Gaussian kernel \cite[Corollary 4.44]{SteChr08}.}. In the following, we nonetheless derive a population expression of Bayes' rule under this strong assumption, use it as a prototype for defining an empirical estimator, and prove its consistency.

\eq{eq:cond_prob_op} has a simple interpretation if the probabilities have density functions and $\pi(x)/p_X(x)$ is in $\Hx$, where $p_X$ is the density function of the marginal $P_X$.  From \eq{eq:mean_integ} we have $m_\Pi(x) = \int k_\cX(x,\tilde{x})\pi(\tilde{x})d\nu_\cX(\tilde{x})= \int k_\cX(x,\tilde{x})(\pi(\tilde{x})/p_X(\tilde{x}))dP_X(\tilde{x})$, which implies $C_{XX}^{-1}m_\Pi = \pi/p_X$ from \eq{eq:cov_op_integ}. Thus \eq{eq:cond_prob_op} is an operator expression of the obvious relation
\[
\int\int k_\cY(y,\tilde{y})p(\tilde{y}|\tilde{x})\pi(\tilde{x})d\nu_\cX(\tilde{x})d\nu_\cY(\tilde{y}) = \int k_\cY(y,\tilde{y})(\pi(\tilde{x})/p_X(\tilde{x}))dP(\tilde{x},\tilde{y}) .
\]

In deriving kernel realization of Bayes' rule, we will use the following tensor representation of the joint probability $Q$, based on Theorem \ref{thm:cond_prob_op}:
\begin{equation}\label{eq:m_Q}
    m_Q = C_{(YX)X} C_{XX}^{-1} m_{\Pi} \in \Hy\otimes\Hx.
\end{equation}
In the above equation, the covariance operator $C_{(YX)X}:\Hx\to\Hy\otimes\Hx$ is defined by the random variable $((Y,X),X)$ taking values on $(\cY\times\cX)\times\cX$.

In many applications of Bayesian inference, the probability conditioned on a particular value should be computed. By plugging the point measure at $x$ into $\Pi$ in \eq{eq:cond_prob_op},
we have a population expression
\begin{equation}\label{eq:cond_mean_x}
E[k_\cY(\cdot,Y)|X=x]=C_{YX}{C_{XX}}^{-1}k_\cX(\cdot,x),
\end{equation}
which has been considered in \cite{Song_etal_ICML2009, Song_etal_AISTATS2010} as the kernel mean of the conditional probability.  It must be noted that for this case the assumption $m_\Pi=k(\cdot,x)\in \mathcal{R}(C_{XX})$ in Theorem \ref{thm:cond_prob_op} may not hold in general\footnote{Suppose
$C_{XX}h_x=k_\cX(\cdot,x)$ were to hold for some $h_x\in\Hx$.  Taking the inner
product with $k_\cX(\cdot,\tilde{x})$ would then imply $k_\cX(x,\tilde{x})=\int
h_x(x')k_\cX(\tilde{x},x')dP_X(x')$, which is not possible for many popular kernels, including the
Gaussian kernel.}.  We will show in Theorem \ref{thm:transition_consistency1}, however, that under some conditions a regularized empirical estimator based on \eq{eq:cond_mean_x} is a consistent estimator of $E[k_\cY(\cdot,Y)|X=x]$.

%The idea in deriving a kernel expression for Bayes' rule relies on the expression of \eq{eq:cond_mean_x}.
If we replace $P$ by $Q$ and $x$ by $y$ in \eq{eq:cond_mean_x}, we obtain
\begin{equation}\label{eq:KBR_population}
    m_{Q_{\cX|y}}=E[k_\cX(\cdot,Z)|W=y]=C_{ZW}C_{WW}^{-1}k_\cY(\cdot,y).
\end{equation}
This is exactly the kernel mean expression of the posterior, and the next step is to provide a way of deriving the covariance operators $C_{ZW}$ and $C_{WW}$. Recall that the kernel mean $m_Q=m_{(ZW)}\in\Hx\otimes\Hy$ can be identified with the covariance operator
$C_{ZW}:\Hy\to\Hx$, and $m_{(WW)}$, which is the kernel mean on the product space $\Hy\otimes\Hy$, with $C_{WW}$.  Then from \eq{eq:m_Q} and the similar expression $m_{(WW)}=C_{(YY)X}C_{XX}^{-1}m_\Pi$, we are able to obtain the operators in \eq{eq:KBR_population}, and thus the kernel mean of the posterior.

The above argument can be rigorously implemented, if empirical estimators are considered.
Let $(X_1,Y_1),\ldots,(X_n,Y_n)$ be an i.i.d.~sample with law $P$.  Since the kernel method needs to express the information of variables in terms of Gram matrices given by data points, we assume that the prior is also expressed in the form of an empirical estimate, and that we have a consistent estimator of $m_\Pi$ in the form
\[
\hml_\Pi = \sum_{j=1}^\ell \gamma_j k_\cX(\cdot,U_j),
\]
where $U_1,\ldots,U_\ell$ are points in $\cX$ and $\gamma_j$ are the weights.  The data points $U_j$ may or may not be a sample from the prior $\Pi$, and negative values are allowed for $\gamma_j$.  Negative values are observed in successive applications of the kernel Bayes rule, as in the state-space example of Section \ref{sec:filtering}.
Based on Theorem \ref{thm:cond_prob_op}, the empirical estimators for $m_{(ZW)}$ and $m_{(WW)}$ are defined respectively by
\begin{equation*}%\label{eq:hat_Q}
    \widehat{m}_{(ZW)}=\hC_{(YX)X}
    \bigl(\hC_{XX}+\eps_n I\bigr)^{-1} \hml_{\Pi}, \quad
    \widehat{m}_{(WW)}=\hC_{(YY)X}
    \bigl(\hC_{XX}+\eps_n I\bigr)^{-1} \hml_{\Pi},
\end{equation*}
where $\eps_n$ is the coefficient of the Tikhonov-type regularization for operator inversion, and $I$ is the identity operator.
The empirical estimators $\widehat{C}_{ZW}$ and $\widehat{C}_{WW}$ for $C_{ZW}$ and $C_{WW}$ are identified with $\widehat{m}_{(ZW)}$ and $\widehat{m}_{(WW)}$, respectively.
In the following, $G_X$ and $G_Y$ denote the Gram matrices $(k_\cX (X_i,X_j))$ and
$(k_\cY(Y_i, Y_j ))$, respectively, and $I_n$ is the identity matrix of size $n$.

\begin{prop}\label{prop:Gram_expr_Q}
The Gram matrix expressions of $\widehat{C}_{ZW}$ and $\widehat{C}_{WW}$ are given by
\[
\widehat{C}_{ZW} =  \sum_{i=1}^n \hmu_i k_\cX(\cdot,X_i)\otimes k_\cY(\cdot,Y_i)\;\;\;\text{and}\;\; \;
\widehat{C}_{WW}= \sum_{i=1}^n \hmu_i k_\cY(\cdot,Y_i)\otimes k_\cY(\cdot,Y_i),
\]
respectively, where the common coefficient $\hmu\in\R^n$ is
\begin{equation}\label{eq:lambda}
\hmu = \Bigl(\frac{1}{n}G_X+\eps_n I_n\Bigr)^{-1} \widehat{{\bf m}}_\Pi,\quad
\widehat{{\bf m}}_{\Pi,i} = \widehat{m}_\Pi(X_i) = \sum_{j =1}^\ell \gamma_j k_\cX(X_i, U_j).
\end{equation}
\end{prop}
The proof is similar to that of Proposition \ref{prop:KBR_Gram} below,  and is omitted.
The expressions in Proposition \ref{prop:Gram_expr_Q} imply that the probabilities $Q$ and
$Q_\cY$ are estimated by the weighted samples $\{((X_i,Y_i), \hmu_i)\}_{i=1}^n$ and $\{(Y_i,\hmu_i)\}_{i=1}^n$, respectively, with common weights.
Since the weight $\hmu_i$ may be negative, in applying \eq{eq:KBR_population} the operator inversion in the form $( \widehat{C}_{WW}+\delta_n I)^{-1}$ may be impossible or unstable.  We thus use another type of Tikhonov regularization, thus obtaining the estimator
\begin{equation}\label{eq:KBR}
    \widehat{m}_{Q_\cX|y} := \widehat{C}_{ZW}\bigl( \widehat{C}_{WW}^2+\delta_n I\bigr)^{-1} \widehat{C}_{WW} k_\cY(\cdot,y).
\end{equation}

\begin{prop}\label{prop:KBR_Gram}
For any $y\in \cY$, the Gram matrix expression of $\widehat{m}_{Q_\cX|y}$ is given by
\begin{equation}\label{eq:KBRemp}
   \widehat{m}_{Q_\cX|y}
   = {\bf k}_X^T R_{X|Y} {\bf k}_Y(y),\qquad R_{X|Y} :=\Lambda G_Y ((\Lambda G_Y)^2 + \delta_n I_n)^{-1}\Lambda,
\end{equation}
where $\Lambda = {\rm diag}(\hmu)$ is a diagonal matrix with
elements $\hmu_i$ in \eq{eq:lambda}, ${\bf k}_X=(k_\cX(\cdot,X_1),\ldots,k_\cX(\cdot,X_n))^T\in \Hx^n$, and ${\bf k}_Y=(k_\cY(\cdot,Y_1),\ldots,k_\cY(\cdot,Y_n))^T\in \Hy^n$.
\end{prop}

\begin{proof}
Let $h=(\widehat{C}_{WW}^2 + \delta_n I)^{-1}\widehat{C}_{WW} k_\cY(\cdot,y)$, and decompose it as $h=\sum_{i=1}^n \alpha_i k_\cY(\cdot,Y_i)+ h_\perp =\alpha^T {\bf k}_Y + h_\perp$, where $h_\perp$ is orthogonal to ${\rm Span}\{k_\cY(\cdot,Y_i)\}_{i=1}^n$.
Expansion of $(\widehat{C}_{WW}^2 + \delta_n I)h = \widehat{C}_{WW} k_\cY(\cdot,y)$ gives ${\bf k}_Y^T (\Lambda G_Y)^2 \alpha + \delta_n {\bf k}_Y^T \alpha + \delta_n h_\perp =  {\bf k}_Y^T \Lambda {\bf k}_Y(y)$.
Taking the inner product with $k_\cY(\cdot,Y_j)$, we have
\[
    \bigl((G_Y\Lambda)^2 + \delta_n I_n\bigr) G_Y\alpha = G_Y\Lambda {\bf k}_Y(y).
\]
The coefficient $\rho$ in $\widehat{m}_{Q_\cX|y}=\widehat{C}_{ZW} h = \sum_{i=1}^n \rho_i k_\cX(\cdot,X_i)$ is given by $\rho=\Lambda G_Y\alpha$, and thus
\[
    \rho = \Lambda \bigl((G_Y\Lambda)^2 + \delta_n I_n\bigr)^{-1}G_Y\Lambda {\bf k}_Y(y) = \Lambda G_Y \bigl((\Lambda G_Y)^2 + \delta_n I_n\bigr)^{-1}\Lambda {\bf k}_Y(y).
\]
\end{proof}

We call Eqs.(\ref{eq:KBR}) and (\ref{eq:KBRemp}) the {\em kernel Bayes' rule} (KBR).  The required computations are summarized in Figure \ref{alg:KBR}.
The KBR uses a weighted sample to represent the posterior; it is similar in this respect to sampling methods such as importance sampling and sequential Monte Carlo (\cite{Docuet_etal_SMC}).  The KBR method, however, does not generate samples of the posterior, but updates the weights of a sample by matrix computation.  We will give some experimental comparisons between KBR and sampling methods in Section \ref{sec:exp_posterior}.

If our aim is to estimate the expectation of a function $f\in\Hx$ with respect to the posterior, the reproducing property \eq{eq:reproducing_mean} gives an estimator
\begin{equation}\label{eq:KBRemp_f}
    \la f,  \widehat{m}_{Q_\cX|y}\ra_{\Hx} = \mathbf{f}_X^T R_{X|Y} \mathbf{k}_\cY(y),
\end{equation}
where $\mathbf{f}_X = (f(X_1),\ldots,f(X_n))^T\in\R^n$.

\begin{figure}[tb]
\hrule
\begin{description}
\item{Input:} (i) $\{(X_i,Y_i)\}_{i=1}^n$: sample to express $P$. (ii) $\{(U_j, \gamma_j)\}_{j=1}^\ell$: weighted sample to express the kernel mean of the prior $\widehat{m}_\Pi$. (iii) $\eps_n, \delta_n$: regularization constants.
\item{Computation:}
\begin{enumerate}
\item Compute Gram matrices $G_X = (k_\cX(X_i,X_j))$, $G_Y = (k_\cY(Y_i,Y_j))$, and a vector $\widehat{{\bf m}}_\Pi = (\sum_{j=1}^\ell \gamma_j k_\cX(X_i, U_j))_{i=1}^n$.
\item Compute $\hmu = n(G_X+n \eps_{n} I_n)^{-1}\widehat{{\bf m}}_\Pi$.
\item Compute $R_{X|Y} = \Lambda G_Y ((\Lambda G_Y)^2 + \delta_{n} I_n)^{-1}\Lambda$, where $\Lambda = {\rm diag}(\hmu)$.
\end{enumerate}
\item{Output:} $n\times n$ matrix $R_{X|Y}$. \\
Given conditioning value $y$, the kernel mean of the posterior $q(x|y)$ is estimated by the weighted sample $\{(X_i, \rho_{i})\}_{i=1}^n$ with weight $\rho=R_{X|Y}{\bf k}_Y(y)$, where ${\bf k}_Y(y) = (k_\cY(Y_i,y))_{i=1}^n$.
\end{description}
\hrule
\vspace*{-2mm}
 \caption{Algorithm of Kernel Bayes' Rule}
  \label{alg:KBR}
\end{figure}

\subsection{Consistency of the KBR estimator}

\label{sec:theory}

We now demonstrate the consistency of the KBR estimator in \eq{eq:KBRemp_f}.  For the theoretical analysis, it is assumed that the distributions have density functions for simplicity.  In the following two theorems, we show only the best rates that can be derived under the assumptions, and defer more detailed discussions and proofs to Section \ref{sec:proof}.  We assume here that the sample size $\ell=\ell_n$ for the prior goes to infinity as the sample size $n$ for the likelihood goes to infinity, and that $\widehat{m}_\Pi^{(\ell_n)}$ is $n^{\alpha}$-consistent in RKHS norm.

%Since the covariance operators $C_{Q}$ and $C_{Q_{\cY\times\cY}}$ are estimated based on the estimator in the above theorem, their best rate in norm is $n^{-\frac{2}{3}\alpha}$.
%The consistency of the KBR method with this best rate is given by the following theorem.

\begin{thm}\label{thm:consitency_KBR_1a}
Let $f$ be a function in $\Hx$, $(Z,W)$ be a random variable on $\cX\times\cY$ such that the distribution is $Q$ with p.d.f.~$p(y|x)\pi(x)$, and $\widehat{m}_\Pi^{(\ell_n)}$ be an
estimator of $m_\Pi$ such that $\|\widehat{m}_\Pi^{(\ell_n)} -
m_\Pi\|_\Hx = O_p(n^{-\alpha})$ as $n\to\infty$ for some $0< \alpha \leq 1/2$.
Assume that $\pi/p_X \in
\mathcal{R}(C_{XX}^{1/2})$, where $p_X$ is the p.d.f.~of $P_X$, and
$E[f(Z)|W=\cdot]\in \mathcal{R}(C_{WW}^2)$.   For the regularization constants
$\eps_n=n^{-\frac{2}{3}\alpha}$ and $\delta_n = n^{- \frac{8}{27}\alpha}$, we have for any $y\in\cY$
\[
     \mathbf{f}^T_X R_{X|Y}\mathbf{k}_Y(y) - E[f(Z)|W=y]
     = O_p(n^{-\frac{8}{27}\alpha}), \quad (n\to\infty),
\]
where $\mathbf{f}_X^T R_{X|Y}\mathbf{k}_Y(y)$ is given by \eq{eq:KBRemp_f}.
\end{thm}

It is possible to extend the covariance operator $C_{WW}$ to  one defined on $L^2(Q_\cY)$ by
\begin{equation}\label{eq:cov_op_L2}
    \tilde{C}_{WW}\phi = \int k_\cY(y,w)\phi(w) dQ_\cY(w), \qquad (\phi\in L^2(Q_\cY)).
\end{equation}
If we consider the convergence on average over $y$, we have a slightly better rate on the consistency of the KBR estimator in $L^2(Q_\cY)$.

\begin{thm}\label{thm:consitency_KBR_2a}
Let $f$ be a function in $\Hx$, $(Z,W)$ be a random vector on $\cX\times\cY$ such that the distribution is $Q$ with p.d.f.~$p(y|x)\pi(x)$, and $\widehat{m}_\Pi^{(\ell_n)}$ be an
estimator of $m_\Pi$ such that $\|\widehat{m}_\Pi^{(\ell_n)} -
m_\Pi\|_\Hx = O_p(n^{-\alpha})$ as $n\to\infty$ for some $0<\alpha \leq 1/2$.
Assume that $\pi/p_X \in
\mathcal{R}(C_{XX}^{1/2})$, where $p_X$ is the p.d.f.~of $P_X$, and
$E[f(Z)|W=\cdot]\in \mathcal{R}(\tilde{C}_{WW}^2)$.     For the regularization constants
$\eps_n=n^{-\frac{2}{3}\alpha}$ and $\delta_n = n^{- \frac{1}{3}\alpha}$, we have
\[
     \bigl\| \mathbf{f}_X^T R_{X|Y}\mathbf{k}_Y(W) - E[f(Z)|W]\bigr\|_{L^2(Q_\cY)}
     = O_p(n^{-\frac{1}{3}\alpha}), \quad (n\to\infty).
\]
\end{thm}

The condition $\pi/p_X \in
\mathcal{R}(C_{XX}^{1/2})$ requires the prior to be sufficiently smooth.
If $\widehat{m}^{(\ell_n)}_\Pi$ is a direct empirical mean with an i.i.d.~sample of size $n$ from $\Pi$, typically $\alpha=1/2$, with which the theorems imply $n^{4/27}$-consistency for every $y$, and $n^{1/6}$-consistency in the $L^2(Q_\cY)$ sense.  While these might seem to be slow rates, the rate of convergence can in practice be much faster than the above theoretical guarantees.
%While the consistency rates shown in the above theorems do not include information of the underlying spaces such as the dimensionality, the rates are not necessarily optimal.  In fact, we can derive faster rates by considering the spectrum of Gram matrices as in \cite{CaponnettoDeVito2007}, on which will provide some discussions in Section \ref{sec:proof}.

%Fukumizu et al. \cite{Fukumizu_etal07kcca}, for instance,  show experimentally that  convergence faster than the theoretical guarantee is observed  for kernel CCA, which uses regularization similar to KBR.  Likewise, Fig. \ref{fig:longseq} demonstrates on dataset (b) that the iterative use of KBR does not enhance the error in filtering over a long sequence.

\section{Bayesian inference with Kernel Bayes' Rule}
\label{sec:KBRmethods}

\subsection{Applications of Kernel Bayes' Rule}
\label{sec:KBR_appl}

In Bayesian inference, we are usually interested in finding a point estimate such as the MAP solution, the expectation of a function under the posterior, or other properties of the distribution.
Given that KBR provides a posterior estimate in the form of a kernel mean (which uniquely determines the distribution when a characteristic kernel is used), we now describe how our kernel approach applies to problems in Bayesian inference.

First,  we have already seen that a consistent estimator  for the expectation of $f\in\Hx$ can be defined with respect to the posterior.  On the other hand, unless $f\in \Hx$ holds, there is no theoretical guarantee that it gives a good estimate.  In Section \ref{sec:exp_posterior}, we discuss some experimental results in such situations.

To obtain a point estimate of the posterior on $x$, it is proposed in \cite{Song_etal_ICML2009}  to use the preimage $\widehat{x}=\arg\min_x\| k_\cX(\cdot,x) - {\bf k}_X^T R_{X|Y} {\bf k}_Y(y)\|^2_{\Hx}$, which represents the posterior mean most effectively by one point.   We use this approach in the present paper when point estimates are considered.  In the case of the Gaussian kernel $\exp(-\|x-y\|^2/(2\sigma^2))$, the fixed point method
\[
    x^{(t+1)} = \frac{\sum_{i=1}^n X_i \rho_i  \exp(-\|X_i-x^{(t)}\|^2/(2\sigma^2))}
    {\sum_{i=1}^n  \rho_i  \exp(-\|X_i-x^{(t)} |^2/(2\sigma^2))},
\]
where $\rho=R_{X|Y} {\bf k}_Y(y)$, can be used
to optimize $x$ sequentially \cite{Mika99kernelpca}.  This method usually converges very fast, although no theoretical guarantee exists for the convergence to the globally optimal point, as is usual in non-convex optimization.

A notable property of KBR is that the prior and likelihood are represented in terms of samples.  Thus, unlike many approaches to Bayesian inference, precise knowledge of the prior and likelihood distributions is not needed, once samples are obtained.  The following are typical situations where the KBR approach is advantageous:
\begin{itemize}
\item The probabilistic relation among variables is difficult to realize with a simple parametric model, while we can obtain samples of the variables easily.  We will see such an example in Section \ref{sec:filtering}.
\item The probability density function of the prior and/or likelihood is hard to obtain explicitly, but sampling is possible:
    \begin{itemize}
    \item In the field of population genetics, Bayesian inference is used with a likelihood expressed by branching processes to model the split of species, for which the explicit density is hard to obtain.   Approximate Bayesian Computation (ABC) is a popular method for approximately sampling from a posterior  without knowing the functional form \citep{Tavare_etal_1997ABC,Marjoram_etal_2003PNAS,Sisson_etal2007}.
    \item
    Another interesting application along these lines is nonparametric Bayesian inference (\cite{MullerQuintana2004} and references therein), in which the prior is typically given in the form of a process without a density form.  In this case, sampling methods are often applied (\cite{MacEachern1994,West_etal1994,MacEachern_etal1999} among others).  Alternatively, the posterior may be approximated using variational methods \cite{BleiJordan2006}.
    \end{itemize}
%    The KBR approach can give alternative ways of Bayesian computation for these problems.
We will present an experimental comparison of KBR and ABC in Section \ref{sec:experiment_ABC}.

\item Even if explicit forms for the likelihood and prior are available, and standard sampling methods such as MCMC or sequential MC are applicable, the computation of a posterior estimate given $y$ might still be computationally costly, making real-time applications unfeasible.  Using KBR, however,  the expectation of a function of the posterior given different $y$ is obtained simply by taking the inner product as in \eq{eq:KBRemp_f}, once $\mathbf{f}_X^T R_{X|Y}$ has been computed.
\end{itemize}

\subsection{Discussions concerning implementation}

When implementing KBR, a number of factors should be borne in mind to ensure good performance. First, in common with many nonparametric approaches, KBR requires training data in the region of the new ``test'' points for results to be meaningful. In other words, if the point on which we condition appears in a region far from the sample used for the estimation, the posterior estimator will be unreliable.

 % Thus, we need sufficient diversity in the samples for reliable estimation of posterior.

Second, in computing the posterior in KBR, Gram matrix inversion is
necessary, which  would cost $O(n^3)$ for sample size $n$ if attempted directly.
Substantial cost reductions can be achieved if the Gram matrices are approximated by low
rank matrix approximations. A popular choice is  the incomplete Cholesky
decomposition \cite{FineScheinberg2001}, which approximates a Gram matrix in the form of $\Gamma\Gamma^T$ with $n\times r$ matrix $\Gamma$ ($r\ll n$) at cost $O(nr^2)$.
Using this and the Woodbury identity, the KBR can be approximately computed at cost $O(nr^2)$.

Third, kernel choice or model selection is key to the effectiveness of any kernel method.
In the case of KBR, we have three model parameters: the kernel (or its parameter, e.g. the bandwidth), the regularization parameter $\eps_n$, and $\delta_n$.  The strategy for parameter selection depends on how the posterior is to be used in the inference problem.  If it is to be applied in regression, we can use standard cross-validation.  In the filtering experiments in Section \ref{sec:experiments}, we use a validation method where we divide the training sample in two.

A more general model selection approach can also be formulated, by creating a new regression problem for the purpose.  Suppose the prior $\Pi$ is given by the marginal $P_X$ of $P$.  The posterior ${Q}_{\cX|y}$ averaged with respect to $P_Y$ is then equal to the marginal $P_X$ itself.  We are thus able to compare the discrepancy of the empirical kernel mean of $P_X$ and the average of the estimators $\widehat{m}_{Q_{\cX|y=Y_i}}$ over $Y_i$.  This leads to a $K$-fold cross validation approach: for a partition of $\{1,\ldots,n\}$ into $K$ disjoint subsets $\{T_a\}_{a=1}^K$, let $\widehat{m}_{Q_{\cX|y}}^{[-a]}$ be the kernel mean of posterior computed using Gram matrices on data $\{(X_i,Y_i)\}_{i\notin T_a}$, and based on the prior mean $\widehat{m}_X^{[-a]}$ with data $\{X_i\}_{i\notin T_a}$.  We can then cross validate by minimizing $\sum_{a=1}^K  \bigl\| \frac{1}{|T_a|} \sum_{j\in T_a} \widehat{m}_{Q_{\cX|y=Y_j}}^{[-a]} - \widehat{m}_X^{[a]}\bigr\|^2_{\Hx}$, where $\widehat{m}_X^{[a]}=\frac{1}{|T_a|} \sum_{j\in T_a}k_\cX(\cdot,X_j)$.

\subsection{Application to nonparametric state-space model}
\label{sec:filtering}

We next describe how KBR may be used in a particular application:  namely, inference in a general time invariant state-space model,
\[
    p(X,Y) = \pi(X_1)\prod_{t=1}^T p(Y_t|X_t) \prod_{t=1}^{T-1} q(X_{t+1}|X_t),
\]
where $Y_t$ is an observable variable, and $X_t$ is a hidden state
variable.  We begin with a brief  review of  alternative strategies for inference in state-space models with complex dynamics, for which linear models are not suitable.
 The extended Kalman filter (EKF) and unscented Kalman filter (UKF, \cite{JulierUhlmann1997}) are nonlinear extensions of the standard linear Kalman filter, and are well established in this setting.
Alternatively,  nonparametric estimates of conditional density functions can be employed, including kernel density estimation or distribution estimates on a partitioning of the space \cite{Monbet_etal2008,Thurn_etal_ICML1999}.  The latter nonparametric approaches are  effective only for low-dimensional cases, however.  Most relevant to this paper are \cite{Song_etal_ICML2009} and \cite{Song_etal_ICML2010}, in which the kernel means and covariance operators are used to implement the nonparametric HMM.

In this paper, we apply the KBR for inference in the nonparametric state-space model.
We do not assume the conditional
probabilities $p(Y_t|X_t)$ and $q(X_{t+1}|X_t)$ to be known explicitly, nor do we estimate them with
simple parametric models. Rather, we assume a sample
$(X_1,Y_1), \ldots, (X_{T+1},Y_{T+1})$ is given for both the
observable and hidden variables in the training phase.
The conditional probability for observation process $p(y|x)$ and the transition $q(x_{t+1}|x_t)$ are
represented by the empirical covariance operators as computed on the training sample,
\begin{align}\label{eq:cor_op_seq}
    & \widehat{C}_{XY}  = \frac{1}{T}\sum_{i=1}^T k_\cX(\cdot,X_i)\otimes k_\cY(\cdot,Y_i),  \quad
    \widehat{C}_{X_{+1}X}  = \frac{1}{T}\sum_{i=1}^T k_\cX(\cdot,X_{i+1})\otimes k_\cX(\cdot,X_i),  \\
    & \widehat{C}_{YY}  = \frac{1}{T}\sum_{i=1}^T k_\cY(\cdot,Y_i)\otimes k_\cY(\cdot,Y_i), \quad
    \widehat{C}_{XX}  = \frac{1}{T}\sum_{i=1}^T k_\cX(\cdot,X_i)\otimes k_\cX(\cdot,X_i). \nonumber
\end{align}

While the sample is not i.i.d., we can use the empirical covariances, which are consistent by the mixing property of Markov models.

Typical applications of the state-space model are filtering, prediction, and smoothing, which are defined by the estimation of $p(x_s|y_1,\ldots,y_t)$ for $s=t$, $s>t$, and $s<t$, respectively. Using the KBR, any of these can be computed.  For simplicity we explain the filtering problem in this paper, but the remaining cases are similar.  In filtering, given new observations $\ty_1,\ldots,\ty_t$, we wish to estimate the current hidden state $x_t$. The sequential estimate for the kernel mean of
$p(x_t|\ty_1,\ldots,\ty_t)$ can be derived via KBR.
Suppose we already have an estimator of the kernel mean of $p(x_t|\ty_1,\ldots,\ty_t)$ in the form
\[
\widehat{m}_{x_t|\ty_1,\ldots,\ty_t}
=  \sum_{i=1}^T \alpha_i^{(t)} k_\cX(\cdot,X_i),
\]
where $\alpha_i^{(t)}=\alpha_i^{(t)}(\ty_1,\ldots,\ty_t)$ are the coefficients at time $t$.

From $p(x_{t+1}| \ty_1,\ldots,\ty_t) = \int p(x_{t+1}|x_t)p(x_t|\ty_1,\ldots,\ty_t)dx_t$, Theorem \ref{thm:cond_prob_op} tells us  the kernel mean of $x_{t+1}$ given $\ty_1,\ldots,\ty_t$ is estimated by $\widehat{m}_{x_{t+1}|\ty_1,\ldots,\ty_t}=\widehat{C}_{X_{+1}X}(\widehat{C}_{XX}+\eps_T I)^{-1} \widehat{m}_{x_t|\ty_1,\ldots,\ty_t} = {\bf k}_{X_{+1}}^T (G_X + T\eps_T I_T)^{-1} G_X \alpha^{(t)}$, where ${\bf k}_{X_{+1}}^T=(k_\cX(\cdot,X_2),\ldots,k_\cX(\cdot,X_{T+1}))$.  Applying Theorem \ref{thm:cond_prob_op} again with $p(y_{t+1}| \ty_1,\ldots,\ty_t) = \int p(y_{t+1}|x_{t+1})p(x_{t+1}|\ty_1,\ldots,\ty_t)dx_t$, we have an estimate for
the kernel mean of the prediction $p(y_{t+1}|\ty_1,\ldots,\ty_t)$,
\begin{equation*}
    \widehat{m}_{y_{t+1}|\ty_1,\ldots,\ty_t} = \widehat{C}_{YX} (\widehat{C}_{XX}+\eps_T I)^{-1} \widehat{m}_{x_{t+1}|\ty_1,\ldots,\ty_t} \\
    = \sum_{i=1}^T \hmu^{(t+1)}_i k_\cY(\cdot,Y_i),
\end{equation*}
where the coefficients $\hmu^{(t+1)}=(\hmu^{(t+1)}_i)_{i=1}^T$ are given by
\begin{equation}\label{eq:update_filter1}
\hmu^{(t+1)}=\bigl(G_X+T\eps_T I_T\bigr)^{-1}G_{X X_{+1}}
\bigl(G_X + T\eps_T I_T\bigr)^{-1} G_X \alpha^{(t)}.
\end{equation}
Here $G_{X X_{+1}}$ is the ``transfer" matrix defined by
$\bigl( G_{XX_{+1}}\bigr)_{ij} = k_\cX(X_{i},X_{j+1})$.  From  $p(x_{t+1}|\ty_1,\ldots,\ty_{t+1})=\frac{p(y_{t+1}|x_{t+1})p(x_{t+1}|\ty_1,\ldots,\ty_{t})}
{ \int p(y_{t+1}|x_{t+1})p(x_{t+1}|\ty_1,\ldots,\ty_{t})dx_{t+1}}$,   kernel Bayes' rule with the prior $p(x_{t+1}|\ty_1,\ldots,\ty_{t})$ and the likelihood $p(y_{t+1}|x_{t+1})$ yields
\begin{equation}\label{eq:update_filter2}
    \alpha^{(t+1)}
    = \Lambda^{(t+1)}G_Y\bigl((\Lambda^{(t+1)}G_Y)^2 +\delta_T I_T\bigr)^{-1}\Lambda^{(t+1)}
    {\bf k}_Y(\ty_{t+1}),
\end{equation}
where
$\Lambda^{(t+1)} = {\rm
diag}(\hmu^{(t+1)}_1,\ldots,\hmu^{(t+1)}_T)$.
Eqs. (\ref{eq:update_filter1}) and (\ref{eq:update_filter2}) describe
the update rule of $\alpha^{(t)}(\ty_1,\ldots,\ty_t)$.

If the prior $\pi(x_1)$ is available, the posterior estimate at $x_1$ given $\ty_1$ is obtained by the kernel Bayes' rule.
If not, we may use \eq{eq:cond_mean_x} to get an  initial estimate  $\widehat{C}_{XY}(\widehat{C}_{YY} + \eps_n I)^{-1} k_\cY(\cdot,\ty_1)$, yielding
$\alpha^{(1)}(\ty_1) =  T(G_Y+T\eps_T I_T)^{-1}{\bf k}_Y(\ty_1)$.

In sequential filtering, a substantial reduction in  computational cost
can be achieved by low
rank matrix approximations, as discussed above.
Given an approximation of rank $r$ for the Gram matrices and transfer matrix,
and
employing the Woodbury identity, the computation costs just
$O(Tr^2)$ for each time step.

\subsection{Bayesian computation without likelihood}
\label{sec:ABC}

We  next address the setting where the likelihood is not known in analytic form, but sampling is possible.  In this case, Approximate Bayesian Computation (ABC) is a popular method for Bayesian inference.
The simplest form of ABC, which is called the rejection method, generates a sample from $q(Z|W=y)$ as follows: (i) generate a sample $X_t$ from the prior $\Pi$,
(ii) generate a sample $Y_t$ from $P(Y|X_t)$,
(iii) if $D(y,Y_t) < \tau$, accept $X_t$; otherwise reject,
(iv) go to (i).
In step (iii), $D$ is a distance measure of the space $\cX$, and $\tau$ is tolerance to acceptance.

In the same setting as ABC,
  KBR gives the following sampling-based method for computing the kernel posterior mean:
\begin{enumerate}
\item Generate a sample $X_1,\ldots,X_n$ from the prior $\Pi$.
\item Generate a sample $Y_t$ from $P(Y|X_t)$ ($t=1,\ldots,n$).
\item Compute Gram matrices $G_X$ and $G_Y$ with $(X_1,Y_1),\ldots,(X_n,Y_n)$, and $R_{X|Y}{\bf k}_Y(y)$.
\end{enumerate}
Alternatively, since $(X_t,Y_t)$ is an sample from $Q$, it is possible to use \eq{eq:cond_mean_x}
for the kernel mean of the conditional probability $q(x|y)$.  As in \cite{Song_etal_ICML2009}, the estimator is given by
\[
    \sum_{t=1}^n \nu_j k_\cX(\cdot,X_t),\quad \nu = (G_Y + N\eps_N I_N)^{-1} {\bf k}_Y(y).
\]

The distribution of a sample generated by ABC approaches to the true posterior if $\tau$ goes to zero, while  empirical estimates via  the kernel approaches converge to the true posterior mean in the limit of infinite sample size.
The efficiency of ABC, however, can be arbitrarily poor for  small $\tau$, since a sample $X_t$ is then rarely accepted in Step (iii).

The ABC method generates a sample, hence any statistics based on the posterior can be approximated.
Given a posterior mean obtained by one of the  kernel methods, however, we may only obtain expectations of functions in the RKHS, meaning that certain statistics (such as confidence intervals) are not straightforward to obtain.
In Section \ref{sec:experiment_ABC}, we present an experimental evaluation of the trade-off between computation time and accuracy for ABC and KBR.

\section{Numerical Examples}
\label{sec:experiments}

\subsection{Nonparametric inference of posterior}
\label{sec:exp_posterior}

The first numerical example is a comparison between KBR and a  kernel density estimation (KDE) approach to obtaining conditional densities.  Let $(X_1,Y_1),\ldots,(X_n,Y_n)$ be an i.i.d.~sample from $P$ on $\R^d\times\R^r$.  With probability density functions $K^\cX(x)$ on $\R^d$ and $K^\cY(y)$ on $\R^r$, the conditional probability density function $p(y|x)$ is estimated by
\[
    \widehat{p}(y|x) = \frac{\sum_{j=1}^n K^\cX_{h_X}(x-X_j) K^\cY_{h_Y}(y-Y_j)}{\sum_{j=1}^n K^\cX_h(x-X_j)},
\]
where $K^\cX_{h_X}(x)=h_X^{-d}K^\cX(x/h_X)$ and $K^\cY_{h_Y}(x)=h_Y^{-r}K^\cY(y/h_Y)$ ($h_X,h_Y>0$).
Given an i.i.d.~sample $U_1,\ldots,U_\ell$ from the prior $\Pi$, the particle representation of the posterior can be obtained by importance weighting (IW). Using this scheme, the posterior $q(x|y)$ given $y\in\R^r$ is represented by the weighted sample $(U_i,\zeta_i)$ with $\zeta_i=\widehat{p}(y|U_i)/\sum_{j=1}^\ell \widehat{p}(y|U_j)$.

We compare the estimates of $\int x q(x|y)dx$ obtained by KBR and KDE + IW, using Gaussian kernels for both the methods.   Note that  the function $f(x)=x$ does not belong to the Gaussian kernel RKHS,
 and the consistency of  KBR is not rigorously guaranteed for this function (c.f.
 Theorem \ref{thm:consitency_KBR_1a}).  That said, Gaussian kernels are known to be able to approximate any continuous function on a compact subset of the Euclidean space with arbitrary accuracy \cite{Steinwart01}.  With such kernels, we can expect the posterior mean to be approximated with high accuracy on any compact set, and thus on average.
In our experiments, the dimensionality was given by $r=d$ ranging from 2 to 64.
The distribution $P$ of $(X,Y)$ was $N((0, {\bf 1}_d^T)^T, V)$ with $V=A^T
A+2 I_d$, where ${\bf 1}_d=(1,\ldots,1)^T\in\R^d$ and each component of $A$ was randomly generated as $N(0,1)$ for each run.  The prior
$\Pi$ was $P_X=N(0,V_{XX}/2)$, where $V_{XX}$ is the $X$-component of $V$.  The
sample sizes were $n = \ell = 200$.  The bandwidth parameters $h_X,h_Y$ in KDE
were set $h_X=h_Y$, and chosen over the set $\{2*i\mid
i=1,\ldots,10\}$ in two ways: least square cross-validation \cite{Rudemo1982,Bowman1984} and the best mean performance.  For the KBR, we chose $\sigma$ in $e^{-\|x-x'\|^2/(2\sigma^2)}$ in two ways: the median over
the pairwise distances in the data \cite{Gretton_etal_nips07}, and the 10-fold cross-validation approach described in Section \ref{sec:KBR_appl}.  Figure \ref{BayesRule} shows the
mean square errors (MSE) of the estimates over 1000 random points
%$\tilde{x}\sim N(0, 2 V_{XX})$ and
$y\sim N(0,V_{YY})$.
 KBR significantly  outperforms the KDE+IW approach. Unsurprisingly, the  MSE of both methods increases with dimensionality.

\begin{figure}
  \centering
  \includegraphics[width=70mm]{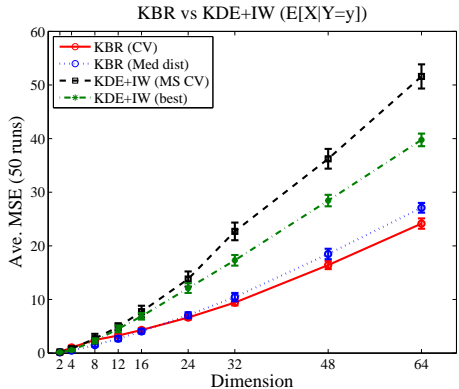}\\
  \caption{Comparison between KBR and KDE+IW.}\label{BayesRule}
\end{figure}

\subsection{Bayesian computation without likelihood}
\label{sec:experiment_ABC}

We compare  ABC and the kernel methods, KBR and conditional mean, in terms of  estimation accuracy and computational time, since they have an obvious tradeoff.  To compute the estimation accuracy rigorously, the ground truth is needed: thus we use Gaussian distributions for the true prior and likelihood, which makes the posterior easy to compute in closed form.
The samples are taken from the same model used in Section \ref{sec:exp_posterior}, and $\int x q(x|y)dx$ is evaluated at 10 different points of $y$. We performed 10 random runs with different random generation of the true distributions.

For ABC, we used only the rejection method; while there are more advanced sampling schemes \cite{Marjoram_etal_2003PNAS,Sisson_etal2007}, their implementation is dependent on the problem being solved.
 Various values for the acceptance region $\tau$ are used, and the accuracy and computational time are shown in Fig.~\ref{fig:ABC} together with total sizes of the generated samples.
For the kernel methods, the sample size $n$ is varied.
 The regularization parameters are given by $\eps_n = 0.01/n$ and $\delta_n = 2\eps_n$ for KBR, and $\eps_n=0.01/\sqrt{n}$ for the conditional kernel mean.  The kernels in the kernel methods are Gaussian kernels for which the bandwidth parameters are chosen by the median of the pairwise distances on the data (\cite{Gretton_etal_nips07}).  The
incomplete Cholesky decomposition is employed for the low-rank approximation.  The results indicate that  kernel methods achieve more accurate results than ABC at a given  computational cost, and the conditional kernel mean shows better results.

\begin{figure}
  \centering
    \includegraphics[width=59mm]{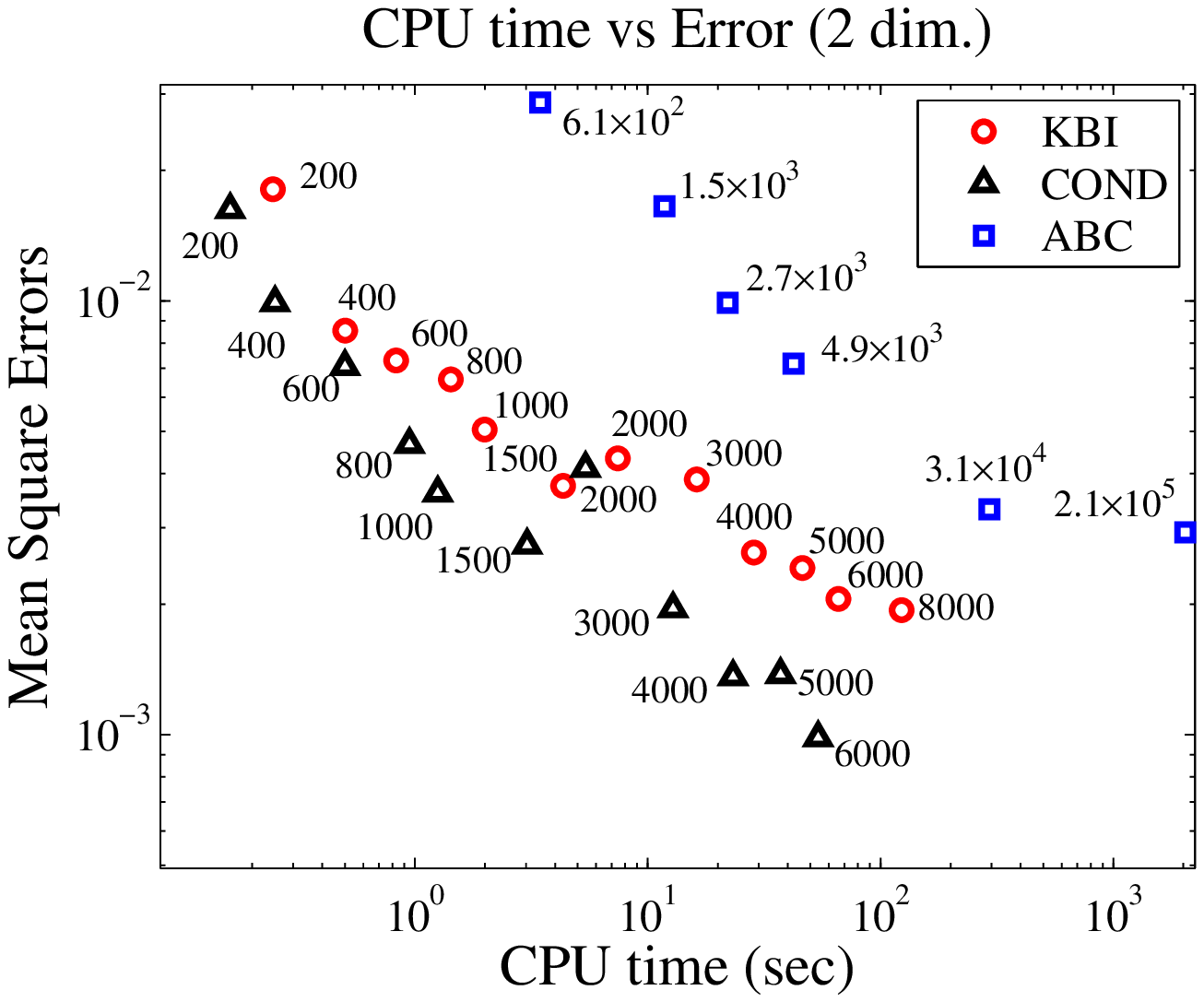}
    \includegraphics[width=60mm]{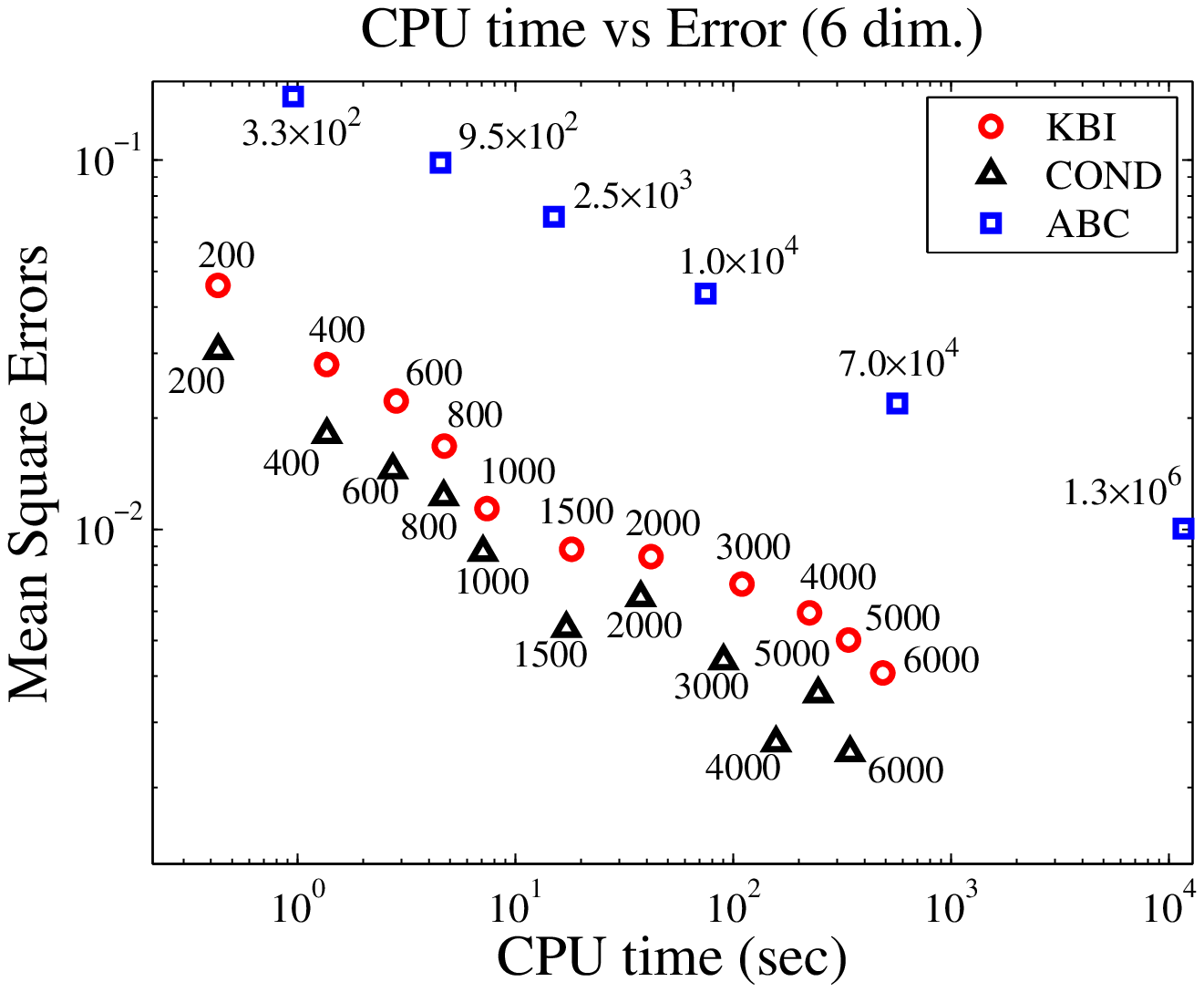}
  \caption{Comparison of estimation accuracy and computational time with KBR and ABC for Bayesian computation without likelihood. The numbers at the marks are the sample sizes generated for computation.}\label{fig:ABC}
\end{figure}

\subsection{Filtering problems}
We next compare the KBR filtering method (proposed in Section \ref{sec:filtering})  with EKF and UKF on synthetic data.

KBR has the regularization parameters
$\eps_T, \delta_T$, and kernel parameters for
$k_\cX$ and $k_\cY$ (e.g., the bandwidth parameter for an RBF kernel).  Under the assumption that a training sample is available, cross-validation can be performed on the training sample to select the parameters. By dividing the training sample into two, one half is used to estimate the covariance operators
\eq{eq:cor_op_seq} with a candidate parameter set, and
the other half to evaluate the estimation errors.
To reduce the search space and attendant computational cost,
we used a simpler procedure, setting $\delta_T = 2\eps_T$, and
using the Gaussian kernel bandwidths $\beta
\sigma_\cX$ and $\beta \sigma_\cY$, where $\sigma_\cX$ and $\sigma_\cY$ are
the median of pairwise distances in the training samples
 (\cite{Gretton_etal_nips07}).
This leaves only two parameters $\beta$ and $\eps_T$ to be tuned.

We applied the KBR filtering algorithm from Section \ref{sec:filtering} to two  synthetic data sets: a simple nonlinear dynamical system, in which the degree of nonlinearity can be controlled, and
the problem of camera orientation recovery from an image sequence.
In the first case, the hidden state is $X_t=(u_t,v_t)^T\in\R^2$,
and the dynamics are given by
\[
 \begin{pmatrix}u_{t+1}\\ v_{t+1}\end{pmatrix} = (1+b\sin(M\theta_{t+1}))\begin{pmatrix}\cos\theta_{t+1} \\ \sin\theta_{t+1}\end{pmatrix} + \zeta_t,
 \quad \theta_{t+1}=\theta_t + \eta\;\;(\text{mod }2\pi),
\]
where
$\eta>0$ is an increment of the angle and $\zeta_t\sim N(0, \sigma_h^2 I_2)$ is independent process noise.  Note that the dynamics of $(u_t,v_t)$ are nonlinear even for $b=0$.  The observation $Y_t$ follows
\[
Y_t = (u_t, v_t)^T + \xi_t,\qquad \xi_t\sim N(0,\sigma_o^2 I),
\]
where $\xi_t$ is independent noise.  The two %three
dynamics are defined as follows.  (a) (rotation with noisy observation) $\eta=0.3$, $b=0$, $\sigma_h = \sigma_o=0.2$.
 (b) (oscillatory rotation with noisy observation) $\eta=0.4$, $b=0.4$, $M=8$, $\sigma_h = \sigma_o=0.2$.  (See Fig.\ref{fig:data}).
% (c) $\eta=0.4$, $b=0$, $\sigma_h=0.2$, $\sigma_o=0.1$.
%The hidden states follows noisy rotation, and the observation mapping $G:\R^2\to\R^{10}$ is given by a three-layer perceptron with parameters generated randomly for each run.

%In the comparison with EKF and UKF,
% for cases (a) and (b)
We assume the correct dynamics are known to the
EKF and UKF.
%Experiment (c) simulates the case where the hidden process is well
%modeled, but it is difficult to model the
%mapping from the hidden state to the observation.  In EKF and UKF
%we estimate the mapping $G$ by a linear model.
The results are shown in
Fig.~\ref{fig:filtering}.  In all the cases, EKF and UKF show unrecognizably small difference.  The dynamics in (a) are weakly nonlinear, and KBR has slightly worse MSE than EKF and UKF. For dataset (b), which has
strong nonlinearity,  KBR outperforms the nonlinear Kalman filter for $T\geq 200$.
%For dataset (c), where EKF uses an incorrect model,  KBR performs better.

\begin{figure}[t]
  \centering
  \includegraphics[height=4.5cm]{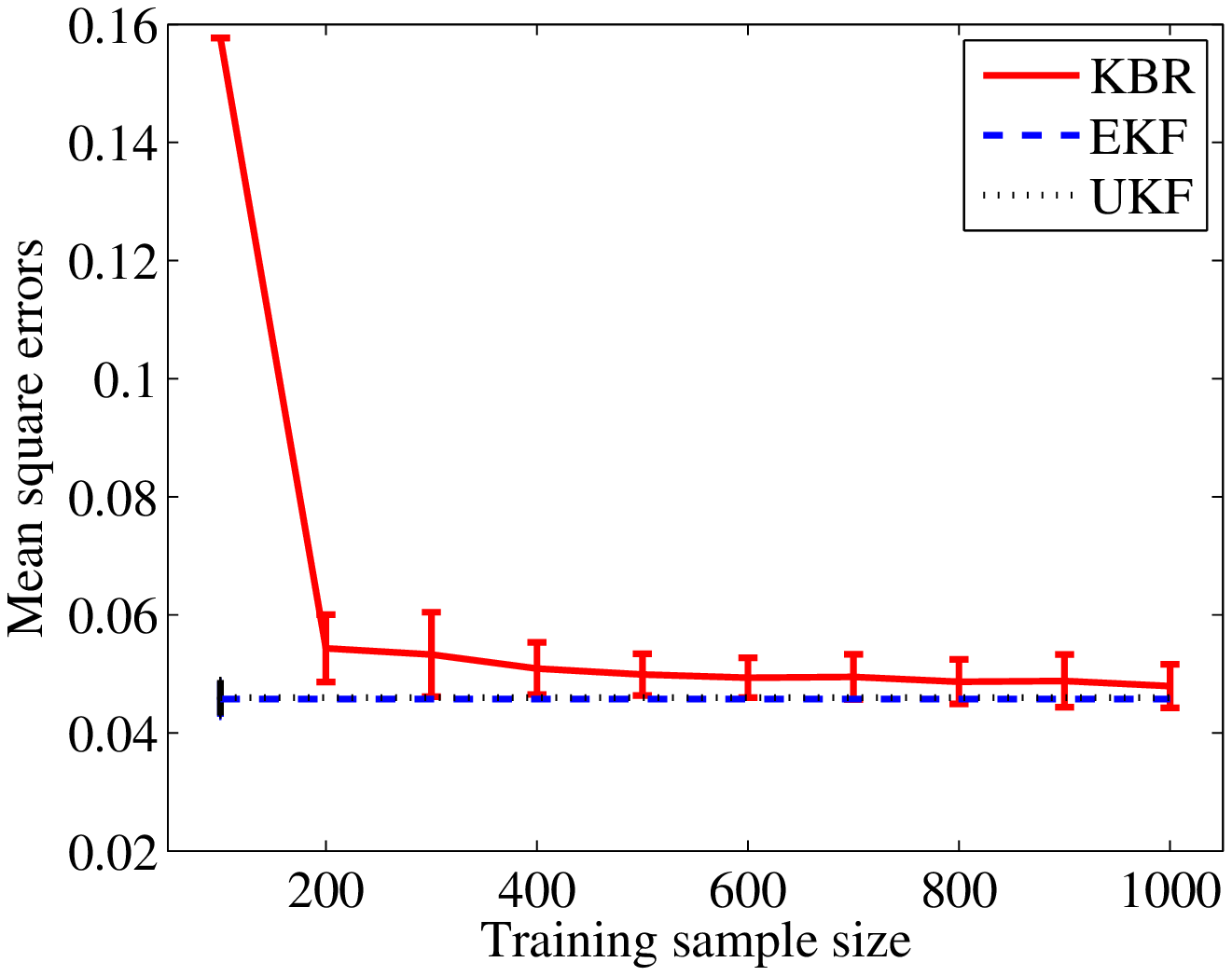}%\hspace*{0.5cm}
  \includegraphics[height=4.5cm]{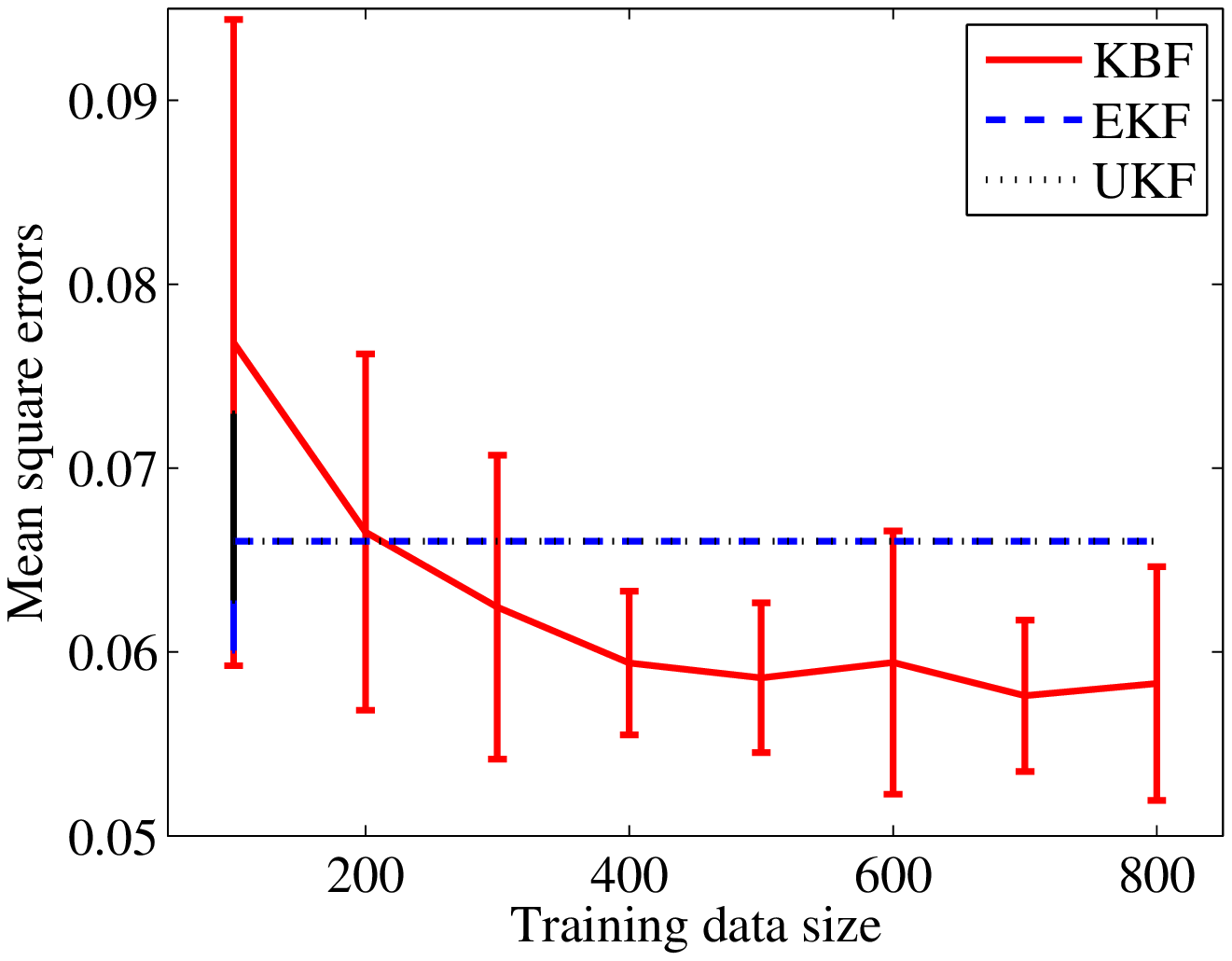}%\hspace*{0.5cm}
\\
  Data (a)\hspace*{5cm}
  Data (b)
  %Results: Data (c)
  \caption{Comparisons with the KBR Filter and EKF. (Average MSEs and standard errors over 30 runs.)
  %The error bars are omitted Data (c), since they are too large to show given the randomness of $G$.
  }\label{fig:filtering}
\end{figure}

\begin{figure}[t]
    \centering
  \includegraphics[height=4cm]{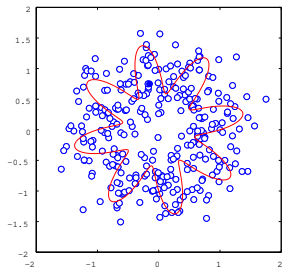}
  \caption{Example of data (b) ($X_t$, $N = 300$)}\label{fig:data}
\end{figure}
%\begin{figure}
%    \centering
%  % Requires \usepackage{graphicx}
%  \includegraphics[height=3.5cm]{graph_longseq.eps}\\
%  \caption{Errors for long sequences ($T=300$)}\label{fig:longseq}
%\end{figure}

In our second synthetic example, we applied the KBR filter to the camera rotation problem used in Song et al. \cite{Song_etal_ICML2009}.  The angle of a camera, which is located at a fixed position, is a hidden variable, and  movie frames recorded  by the camera are observed.  The data are generated virtually using a computer graphics environment.  As in \cite{Song_etal_ICML2009}, we are given 3600 downsampled frames of $20\times 20$ RGB pixels ($Y_t\in[0,1]^{1200}$), where the first 1800 frames are used for training, and the second half are used to test the filter.  We make the data noisy by adding Gaussian noise $N(0,\sigma^2)$ to $Y_t$.

Our experiments cover two settings.  In the first, we assume we do not know that the hidden state $S_t$ is included in $SO(3)$, but only that it is a general $3\times 3$ matrix.  In this case, we use the Kalman filter by estimating the relations under a linear assumption, and the KBR filter with Gaussian kernels for $S_t$ and $X_t$ as Euclidean vectors.  In the second setting, we exploit the fact that $S_t\in SO(3)$: for the Kalman Filter, $S_t$ is represented by a quanternion, which is a standard vector representation of rotations; for the KBR filter the kernel $k(A,B)={\rm Tr}[AB^T]$ is used for $S_t$, and $S_t$ is estimated within $SO(3)$. Table \ref{tbl:camera} shows the Frobenius norms between the estimated matrix and the true one.
The KBR filter significantly outperforms the EKF, since KBR has the advantage in extracting the complex nonlinear dependence between the observation and  the hidden state.

\begin{table}
  \centering
    \begin{tabular}{c|cc|cc}
       \hline
         & KBR (Gauss) & KBR (Tr) & Kalman (9 dim.) &  Kalman (Quat.)  \\
       \hline
       %$\sigma^2=10^{-4}$ & $0.187$ & $0.122$ & $3.800$ &  $1.044$ \\
       $\sigma^2=10^{-4}$ & $0.210\pm 0.015$ & $0.146\pm 0.003$  & $1.980\pm0.083$ & $0.557\pm 0.023$ \\
       $\sigma^2=10^{-3}$ & $0.222\pm 0.009$ & $0.210\pm 0.008$ & $1.935\pm0.064$ &  $0.541\pm 0.022$ \\
       \hline
     \end{tabular}
  \caption{Average MSE and standard errors of estimating camera angles (10 runs).}\label{tbl:camera}
\end{table}

\section{Proofs}
\label{sec:proof}

The proof idea for the consistency rates of the KBR estimators is similar to
\cite{CaponnettoDeVito2007,SmaleZhou2005}, in which the basic techniques are taken from the general
theory of regularization \cite{EnglHankeNeubauer}.

The first preliminary result is a rate of convergence for the mean transition in Theorem \ref{thm:cond_prob_op}. In the following $\mathcal{R}(C_{XX}^0)$ means $\Hx$.

\begin{thm}\label{thm:transition_consistency1}
Assume that $\pi/p_X \in
\mathcal{R}(C_{XX}^\beta)$ for some $\beta\geq 0$, where $\pi$ and $p_X$ are the p.d.f.~of $\Pi$ and $P_X$, respectively.
Let $\widehat{m}_\Pi^{(n)}$ be an
estimator of $m_\Pi$ such that $\|\widehat{m}_\Pi^{(n)} -
m_\Pi\|_\Hx = O_p(n^{-\alpha})$ as $n\to\infty$ for some $0<\alpha \leq 1/2$.  Then, with $\eps_n=n^{-\max\{\frac{2}{3}\alpha, \frac{\alpha}{1+\beta}\}}$, we have
\[
    \bigl\| \hC_{YX}\bigl( \hC_{XX}+\eps_n I\bigr)^{-1} \widehat{m}_\Pi^{(n)}
    - m_{Q_\cY} \bigr\|_\Hy = O_p(n^{-\min\{\frac{2}{3}\alpha, \frac{2\beta+1}{2\beta+2}\alpha\}}),\quad (n\to\infty).
\]
\end{thm}

\begin{proof}
Take $\eta\in\Hx$ such that $\pi/p_X =C_{XX}^\beta\eta$.  Then, we have
\begin{equation}\label{eq:m_pi_integ}
m_\Pi
= \int k_\cX(\cdot,x)\frac{\pi(x)}{p_X(x)}p_X(x)d\nu_\cX(x) = C_{XX}^{\beta+1}\eta.
\end{equation}

First we show the rate of the estimation error:
\begin{equation}\label{eq:Qy_estimation}
\bigl\| \hC_{YX}\bigl( \hC_{XX}+\eps_n I\bigr)^{-1} \hm_\Pi
   - C_{YX}\bigl( C_{XX}+\eps_n I\bigr)^{-1} m_\Pi \bigr\|_\Hy
   =O_p\bigl(n^{-\alpha}\eps_n^{-1/2}\bigr),
\end{equation}
as $n\to\infty$.  By using  $B^{-1}-A^{-1} = B^{-1}(A-B)A^{-1}$  for any invertible operators $A$ and $B$, the left hand side of \eq{eq:Qy_estimation} is upper
bounded by
\begin{multline*}
\bigl\|\hC_{YX}\bigl( \hC_{XX}+\eps_n I\bigr)^{-1}
\bigl(\hm_\Pi - m_\Pi\bigr)\bigr\|_{\Hy} +\bigl\|  \bigl(\hC_{YX}-C_{YX}\bigr)
\bigl(C_{XX}+\eps_n I\bigr)^{-1}m_\Pi \bigr\|_{\Hy} \\
 +  \bigl\| \hC_{YX}\bigl( \hC_{XX}+\eps_n I\bigr)^{-1}
\bigl(C_{XX}-\hC_{XX}\bigr) \bigl(C_{XX}+\eps_n I\bigr)^{-1}m_\Pi \bigr\|_{\Hy}.
\end{multline*}
By the decomposition $\hC_{YX}=\widehat{C}_{YY}^{(n)1/2} \widehat{W}_{YX}^{(n)}\widehat{C}_{XX}^{(n)1/2}$ with $\|\widehat{W}_{YX}^{(n)}\|\leq 1$ \cite{Baker73}, we have $\|\hC_{YX}\bigl( \hC_{XX}+\eps_n I\bigr)^{-1}\|=O_p(\eps_n^{-1/2})$, which implies the first term is of $O_p(n^{-\alpha}\eps_n^{-1/2})$.  From the $\sqrt{n}$ consistency of the covariance operators and $m_\Pi=C_{XX}^{\beta+1}\eta$, a similar argument to the first term proves that the second and third terms are of the
order $O_p(n^{-1/2})$ and $O_p(n^{-1/2}\eps_n^{-1/2})$,
respectively, which means \eq{eq:Qy_estimation}.

Next, we show the rate for the approximation error
\begin{equation}\label{eq:Qy_approx}
\bigl\|C_{YX}\bigl( C_{XX}+\eps_n I\bigr)^{-1} m_\Pi - m_{Q_\cY}\bigr\|_\Hy
= O(\eps_n^{\min\{(1+2\beta)/2, 1\}})\qquad (n\to\infty).
\end{equation}
Let $C_{YX}=C_{YY}^{1/2}W_{YX}C_{XX}^{1/2}$ be the decomposition with $\|W_{YX}\|\leq 1$.
It follows from \eq{eq:m_pi_integ} and the relation
\[
    m_{Q_\cY} = \int\int k(\cdot,y)\frac{\pi(x)}{p_X(x)}p(x,y)
    d\nu_\cX(x)d\nu_\cY(y) = C_{YX}C_{XX}^\beta \eta
\]
that the left hand side of \eq{eq:Qy_approx} is upper bounded by
\begin{equation*}
 \| C_{YY}^{1/2}W_{YX} \|\,\|\bigl( C_{XX} + \eps_n I\bigr)^{-1} C_{XX}^{(2\beta+3)/2}\eta - C_{XX}^{(2\beta+1)/2}\eta\|_\Hx.
\end{equation*}
By the eigendecomposition $C_{XX}=\sum_i \lambda_i \phi_i \la \phi_i,\cdot\ra$, where $\{\lambda_i\}$ are the positive eigenvalues and $\{\phi_i\}$ are the corresponding unit eigenvectors, the expansion
\begin{align*}
 \bigl\|\bigl( C_{XX} + \eps_n I\bigr)^{-1} C_{XX}^{(2\beta+3)/2}\eta - C_{XX}^{(2\beta+1)/2}\eta\bigr\|_\Hx^2
= \sum_i \biggl( \frac{\eps_n \lambda_i^{(2\beta+1)/2}}{\lambda_i + \eps_n} \biggr)^2 \la \eta,\phi_i\ra^2
\end{align*}
holds. If $0\leq \beta < 1/2$, we have $\frac{ \eps_n\lambda_i^{(2\beta+1)/2} }{ \lambda_i+\eps_n }= \frac{ \lambda_i^{(2\beta+1)/2} }{ (\lambda_i+\eps_n)^{(2\beta+1)/2} } \frac{ \eps_n^{(1 - 2\beta)/2} }{ (\lambda_i+\eps_n)^{(1-2\beta)/2} }\eps_n^{(2\beta+1)/2}\leq \eps_n^{(2\beta+1)/2}$.   If $\beta \geq 1/2$, then $\frac{ \eps_n\lambda_i^{(2\beta+1)/2} }{ \lambda_i+\eps_n } \leq \|C_{XX}\|\eps_n$.  The dominated convergence theorem shows that the the above sum converges to zero of the order $O(\eps_n^{\min\{2\beta+1, 2\}})$ as $\eps_n\to0$.

From Eqs. (\ref{eq:Qy_estimation}) and (\ref{eq:Qy_approx}), the optimal order of $\eps_n$ and the optimal rate of consistency are  given as claimed.
\end{proof}

The following theorem shows the consistency rate of the estimator used in the conditioning step \eq{eq:KBR_population}.
\begin{thm}\label{thm:consitency_conditioning1}
Let $f$ be a function in $\Hx$, and $(Z,W)$ be a random variable taking values in $\cX\times\cY$.  Assume that $E[f(Z)|W=\cdot]\in \mathcal{R}(C_{WW}^\nu)$ for some $\nu \geq 0$, and $\hC_{WZ}:\Hx\to\Hy$ and $\hC_{WW}:\Hy\to\Hy$ be compact operators, which may not be positive definite, such that $\| \hC_{WZ}-C_{WZ}\|=O_p(n^{-\gamma})$ and $\| \hC_{WW}-C_{WW}\|=O_p(n^{-\gamma})$ for some $\gamma > 0$.  Then, for a positive sequence $\delta_n = n^{-\max\{ \frac{4}{9}\gamma,\frac{4}{2\nu+5}\gamma\}}$, we have as $n\to\infty$
\[
\bigl\| \hC_{WW}\bigl( (\hC_{WW})^2 + \delta_n I \bigr)^{-1} \hC_{WZ} f -E[f(X)|W=\cdot]\bigr\|_{\Hx} = O_p( n^{-\min\{\frac{4}{9}\gamma,\frac{2\nu}{2\nu+5}\gamma\}}).
\]
\end{thm}
\begin{proof}
Let $\eta\in\Hx$ such that $E[f(Z)|W=\cdot]=C_{WW}^\nu \eta$. First we show
\begin{multline}\label{eq:estim_order}
\bigl\| \hC_{WW}\bigl((\hC_{WW})^2+\delta_n I\bigr)^{-1} \hC_{WZ}f
- C_{WW} (C_{WW}^2+\delta_n I)^{-1}C_{WZ}f \bigr\|_\Hx  \\ =O_p(n^{-\gamma}\delta_n^{-5/4}).
\end{multline}
The left hand side of \eq{eq:estim_order} is upper bounded by
\begin{align*}
& \bigl\| \hC_{WW}\bigl((\hC_{WW})^2+\delta_n I\bigr)^{-1} (\hC_{WZ}-C_{WZ})f
\bigr\|_\Hy  \\
& +
\bigl\| (\hC_{WW}-C_{WW})(C_{WW}^2+\delta_n I)^{-1}C_{WZ}f
\bigr\|_\Hy   \\
& + \bigl\| \hC_{WW}( (\hC_{WW})^2+\delta_n I\bigr)^{-1} \bigl( (\hC_{WW})^2 - C_{WW}^2\bigr)
\bigl(C_{WW}^2+\delta_n I\bigr)^{-1} C_{WZ}f
\bigr\|_\Hy.
\end{align*}
Let $\hC_{WW}= \sum_i \lambda_i \phi_i\la \phi_i,\cdot\ra$ be the eigendecomposition, where $\{\phi_i\}$ is the unit eigenvectors and $\{\lambda_i\}$ is the corresponding eigenvalues. From $\bigl| \lambda_i/(\lambda_i^2 + \delta_n)\bigr| = 1/|\lambda_i + \delta_n/\lambda_i| \leq 1/(2\sqrt{|\lambda_i|}\sqrt{\delta_n/|\lambda_i|}) = 1/(2\sqrt{\delta_n})$, we have $\| \hC_{WW}\bigl((\hC_{WW})^2+\delta_n I\bigr)^{-1}\| \leq 1/(2\sqrt{\delta_n})$, and thus the first term of the above bound is of $O_p(n^{-\gamma}\delta_n^{-1/2})$.  A similar argument by the eigendecomposition of $C_{WW}$ combined with the decomposition  $C_{WZ}=C_{WW}^{1/2}U_{WZ}C_{ZZ}^{1/2}$ with $\| U_{WZ} \|\leq 1$ shows that the second term is of $O_p(n^{-\gamma}\delta_n^{-3/4})$. From the fact $\|(\hC_{WW})^2-C_{WW}^2\|\leq \| \hC_{WW}(\hC_{WW}-C_{WW})\| + \| (\hC_{WW}-C_{WW})C_{WW}\| = O_p(n^{-\gamma})$, the third term is of $O_p(n^{-\gamma}\delta_n^{-5/4})$.  This implies \eq{eq:estim_order}.

From $E[f(Z)|W=\cdot] = C_{WW}^\nu\eta$ and $C_{WZ}f = C_{WW}E[f(Z)|W=\cdot]=C_{WW}^{\nu+1}\eta$, the convergence rate
\begin{equation}\label{eq:app_order}
\bigl\| C_{WW} (C_{WW}^2+\delta_n I)^{-1}C_{WZ}f - E[f(Z)|W=\cdot]\bigr\|_\Hy = O(\delta_n^{\min\{1, \frac{\nu}{2} \}}).
\end{equation}
can be proved by the same way as \eq{eq:Qy_approx}.

Combination of Eqs.(\ref{eq:estim_order}) and (\ref{eq:app_order}) proves the assertion.
\end{proof}

Recall that $\tilde{C}_{WW}$ is the integral operator on $L^2(Q_\cY)$ defined by \eq{eq:cov_op_L2}.
The following theorem shows the consistency rate on average.
Here $\mathcal{R}(\tilde{C}_{WW}^0)$ means $L^2(Q_\cY)$.

\begin{thm}\label{thm:consitency_conditioning2}
Let $f$ be a function in $\Hx$, and $(Z,W)$ be a random variable taking values in $\cX\times\cY$ with distribution
$Q$.  Assume that $E[f(Z)|W=\cdot]\in \mathcal{R}(\tilde{C}_{WW}^\nu)\cap\Hy$ for some $\nu > 0$, and $\hC_{WZ}:\Hx\to\Hy$ and $\hC_{WW}:\Hy\to\Hy$ be compact operators, which may not be positive definite, such that $\| \hC_{WZ}-C_{WZ}\|=O_p(n^{-\gamma})$ and $\| \hC_{WW}-C_{WW}\|=O_p(n^{-\gamma})$ for some $\gamma > 0$.  Then, for a positive sequence $\delta_n = n^{-\max\{ \frac{1}{2}\gamma,\frac{2}{\nu+2}\gamma\}}$, we have as $n\to\infty$
\[
\bigl\| \hC_{WW}\bigl( (\hC_{WW})^2 + \delta_n I \bigr)^{-1} \hC_{WZ} f -E[f(X)|W=\cdot]\bigr\|_{L^2(Q_\cY)} = O_p( n^{-\min\{\frac{1}{2}\gamma,\frac{\nu}{\nu+2}\gamma\}}).
\]
\end{thm}
\begin{proof}
Note that for $f,g\in\Hx$ we have $(f,g)_{L^2(Q_\cY)} = E[f(W)g(W)] = \la f,C_{WW}g\ra_\Hx$.  It follows that the left hand side of the assertion is equal to
\[
\bigl\| C_{WW}^{1/2}\hC_{WW}\bigl((\hC_{WW})^2+\delta_n I\bigr)^{-1} \hC_{WZ}f - C_{WW}^{1/2}E[f(Z)|W=\cdot]\bigr\|_{\Hy}.
\]

First, by the similar argument to the proof of \eq{eq:estim_order}, it is easy to show that the rate of the estimation error is given by
\begin{multline*}%\label{eq:estim_rate2}
\bigl\| C_{WW}^{1/2}\bigl\{ \hC_{WW}\bigl((\hC_{WW})^2+\delta_n I\bigr)^{-1} \hC_{WZ}f - C_{WW}(C_{WW}^2+\delta_n I)^{-1}C_{WZ}f \bigr\} \bigr\|_{\Hy} \\
=O_p(n^{-\gamma }\delta_n^{-1}).
\end{multline*}
%In fact, the left hand side is upper bounded by
%\begin{align*}
% & \bigl\| C_{WW}^{1/2}(\hC_{WW}-C_{WW})\bigl( (\hC_{WW})^2+\delta_n I\bigr)^{-1} \hC_{WZ}f \bigr\|_\Hy \\
%& \quad +
%\bigl\| C_{WW}^{3/2}\bigl(C_{WW}^2+\delta_n I\bigr)^{-1} (\hC_{WZ}-C_{WZ})f \bigr\|_\Hy  \\
%& \quad + \bigl\| C_{WW}^{3/2}\bigl(C_{WW}^2+\delta_n I\bigr)^{-1}\bigl( (\hC_{WW})^2 - C_{WW}^2\bigr)\bigl( (\hC_{WW})^2+\delta_n I\bigr)^{-1}
% \hC_{WZ}f\bigr\|_\Hy.
%\end{align*}
%the above terms are of the order $O_p(n^{-\gamma}\delta^{-3/4})$, $O_p(n^{-\gamma}\delta_n^{-1/4})$, and $O_p(n^{-\gamma} \delta_n^{-1})$, respectively. This implies \eq{eq:estim_rate2}.
%
It suffices then to prove
\begin{equation*}%\label{eq:app_order2}
\bigl\| C_{WW} (C_{WW}^2+\delta_n I)^{-1}C_{WZ}f - E[f(Z)|W=\cdot]\bigr\|_{L^2(Q_\cY)} = O(\delta_n^{\min\{1, \frac{\nu}{2} \}}).
\end{equation*}
Let $\xi\in L^2(Q_\cY)$ such that $E[f(Z)|W=\cdot]=\tilde{C}_{WW}^\nu \xi$.  In a similar way to Theorem \ref{thm:cond_mean}, $\tilde{C}_{WW}E[f(Z)|W] = \tilde{C}_{WZ}f$ holds, where $\tilde{C}_{WZ}$ is the extension of $C_{WZ}$, and thus $C_{WZ}f = \tilde{C}_{WW}^{\nu+1}\xi$.  The left hand side of the above equation is equal to
\[
    \bigl\| \tilde{C}_{WW} (\tilde{C}_{WW}^2+\delta_n I)^{-1}\tilde{C}_{WW}^{\nu+1}\xi - \tilde{C}_{WW}^{\nu}\xi\bigr\|_{L^2(Q_cY)}.
\]
By the eigendecomposition of $\tilde{C}_{WW}$ in $L^2(Q_\cY)$, a similar argument to the proof of \eq{eq:app_order} shows the assertion.
\end{proof}

The consistency of KBR follows by combining the above theorems.
\begin{thm}\label{thm:consitency_KBR1}
Let $f$ be a function in $\Hx$, $(Z,W)$ be a random variable that has the distribution $Q$ with p.d.f.~$p(y|x)\pi(x)$, and $\widehat{m}_\Pi^{(n)}$ be an
estimator of $m_\Pi$ such that $\|\widehat{m}_\Pi^{(n)} -
m_\Pi\|_\Hx = O_p(n^{-\alpha})$ ($n\to\infty$) for some $0<\alpha\leq 1/2$. Assume that $\pi/p_X \in
\mathcal{R}(C_{XX}^{\beta})$ with $\beta\geq 0$,  and  $E[f(Z)|W=\cdot]\in \mathcal{R}(C_{WW}^\nu)$ for some $\nu\geq0$.  For the regularization constants
$\eps_n=n^{-\max\{\frac{2}{3}\alpha,\frac{1}{1+\beta}\alpha\}}$ and $\delta_n = n^{-\max\{ \frac{4}{9}\gamma,\frac{4}{2\nu+5}\gamma\}}$, where $\gamma=\min\{ \frac{2}{3}\alpha,\frac{2\beta+1}{2\beta+2}\alpha\}$, we have for any $y\in\cY$
\[
     \mathbf{f}^T_X R_{X|Y}\mathbf{k}_Y(y) - E[f(Z)|W=y]
     = O_p(n^{-\min\{\frac{4}{9}\gamma,\frac{2\nu}{2\nu+5}\gamma\}}), \quad (n\to\infty),
\]
where $\mathbf{f}_X^T R_{X|Y}\mathbf{k}_Y(y)$ is given by \eq{eq:KBRemp}.
\end{thm}
\begin{proof}
By applying Theorem \ref{thm:transition_consistency1} to
$Y=(Y,X)$ and $Y=(Y,Y)$, we see that both of $\|\widehat{C}_{WZ}-C_{WZ}\|$
and $\|\widehat{C}_{WW}-C_{WW}\|$ are of
$O_p(n^{-\gamma})$. Since
\begin{multline*}
    \mathbf{f}^T_X R_{X|Y}\mathbf{k}_Y(y) - E[f(Z)|W=y]  \\
    = \la k_\cY(\cdot,y), \widehat{C}_{WW}\bigl( (\widehat{C}_{YY})^2 + \delta_n I \bigr)^{-1} \widehat{C}_{WZ} f - E[f(Z)|W=\cdot]\ra_\Hy,
\end{multline*}
combination of Theorems \ref{thm:transition_consistency1} and \ref{thm:consitency_conditioning1} proves the theorem.
\end{proof}

The next theorem shows the rate on average w.r.t.~$Q_\cY$. The proof is similar to the above theorem, and omitted.
\begin{thm}\label{thm:consitency_KBR2}
Let $f$ be a function in $\Hx$, $(Z,W)$ be a random variable that has the distribution $Q$ with p.d.f.~$p(y|x)\pi(x)$, and $\widehat{m}_\Pi^{(n)}$ be an
estimator of $m_\Pi$ such that $\|\widehat{m}_\Pi^{(n)} -
m_\Pi\|_\Hx = O_p(n^{-\alpha})$ ($n\to\infty$) for some $0<\alpha \leq 1/2$. Assume that $\pi/p_X \in
\mathcal{R}(C_{XX}^{\beta})$ with $\beta\geq 0$,  and  $E[f(Z)|W=\cdot]\in \mathcal{R}(\tilde{C}_{WW}^\nu)\cap\Hy$ for some $\nu >0$.   For the regularization constants
$\eps_n=n^{-\max\{\frac{2}{3}\alpha,\frac{1}{1+\beta}\alpha\}}$ and $\delta_n = n^{-\max\{ \frac{1}{2}\gamma,\frac{2}{\nu+2}\gamma\}}$, where $\gamma=\min\{ \frac{2}{3}\alpha,\frac{2\beta+1}{2\beta+2}\alpha\}$, we have
\[
     \bigl\| \mathbf{f}_X^T R_{X|Y}\mathbf{k}_Y(W) - E[f(Z)|W]\bigr\|_{L^2(Q_\cY)}
     = O_p(n^{-\min\{\frac{1}{2}\gamma,\frac{\nu}{\nu+2}\gamma\}}), \quad (n\to\infty).
\]
\end{thm}

We  also have consistency of the estimator for the kernel mean of posterior $m_{Q_{\cX|y}}$, if we make stronger assumptions. First, we formulate the expectation with the posterior in terms of operators.  Let $(Z,W)$ be a random variable with distribution $Q$. Assume that for any $f\in\Hx$ the conditional expectation $E[f(Z)|W=\cdot]$ is included in $\Hy$.  We  then have a linear operator $S$ defined by
\[
S:\Hx\to\Hy, \qquad f\mapsto E[f(Z)|W=\cdot].
\]
If we further assume that $S$ is bounded, the adjoint operator $S^*:\Hy\to\Hx$ satisfies
\[
    \la S^*k_\cY(\cdot,y), f\ra_\Hx = \la k_\cY(\cdot,y), Sf\ra_\Hy = E[f(Z)|W=y]
\]
for any $y\in\cY$, and thus $S^*k_\cY(\cdot,y)$ is equal to the kernel mean of the conditional probability of $Z$ given $W=y$.

We make the following further assumptions:\\
{\bf Assumption (S)}
\begin{enumerate}
\item The covariance operator $C_{WW}$ is injective.
\item There exists $\nu> 0$ such that for any $f\in \Hx$ there is $\eta_f\in\Hx$ with  $Sf = C_{WW}^\nu\eta_f$, and the linear map
\[
    C_{WW}^{-\nu}S: \Hx\to\Hy,\qquad  f\mapsto \eta_f
\]
is bounded.
\end{enumerate}

\begin{thm}\label{thm:consitency_KBR3}
Let $(Z,W)$ be a random variable that has the distribution $Q$ with p.d.f.~$p(y|x)\pi(x)$, and $\widehat{m}_\Pi^{(n)}$ be an
estimator of $m_\Pi$ such that $\|\widehat{m}_\Pi^{(n)} -
m_\Pi\|_\Hx = O_p(n^{-\alpha})$ ($n\to\infty$) for some $0<\alpha \leq 1/2$.  Assume (S) above, and $\pi/p_X \in
\mathcal{R}(C_{XX}^{\beta})$ with some $\beta\geq 0$.  For the regularization constants
$\eps_n=n^{-\max\{\frac{2}{3}\alpha,\frac{1}{1+\beta}\alpha\}}$ and $\delta_n = n^{-\max\{ \frac{4}{9}\gamma,\frac{4}{2\nu+5}\gamma\}}$, where $\gamma=\min\{ \frac{2}{3}\alpha,\frac{2\beta+1}{2\beta+2}\alpha\}$, we have for any $y\in\cY$
\[
     \bigl\| \mathbf{k}_X^T R_{X|Y}\mathbf{k}_Y(y) - m_{Q_\cX|y} \bigr\|_{\Hx}
     = O_p(n^{-\min\{\frac{4}{9}\gamma,\frac{2\nu}{2\nu+5}\gamma\}}),
\]
as $n\to\infty$, where $m_{Q_\cX|y}$ is the kernel mean of the posterior given $y$.
\end{thm}
\begin{proof}
First, in a similar manner to the proof of \eq{eq:estim_order}, we have
\begin{multline*}
\bigl\| \hC_{ZW}\bigl((\hC_{WW})^2+\delta_n I\bigr)^{-1} \hC_{WW}k_\cY(\cdot,y)
- C_{ZW} (C_{WW}^2+\delta_n I)^{-1}C_{WW}k_\cY(\cdot,y) \bigr\|_\Hx  \\ =O_p(n^{-\gamma}\delta_n^{-5/4}).
\end{multline*}
The assertion is thus obtained if
\begin{equation}\label{eq:app_order3}
\bigl\| C_{ZW} (C_{WW}^2+\delta_n I)^{-1}C_{WW}k_\cY(\cdot,y) - S^*k_\cY(\cdot,y)\bigr\|_\Hx = O(\delta_n^{\min\{1, \frac{\nu}{2} \}})
\end{equation}
is proved. The left hand side of \eq{eq:app_order3} is upper-bounded by
\begin{multline*}
\bigl\| C_{ZW} (C_{WW}^2+\delta_n I)^{-1}C_{WW}- S^*\|\,\|k_\cY(\cdot,y)\|_\Hy  \\
= \bigl\| C_{WW}(C_{WW}^2+\delta_n I)^{-1}C_{WZ} -S\bigr\|\,\|k_\cY(\cdot,y)\|_\Hy.
\end{multline*}
It follows from Theorem \ref{thm:cond_mean} that $C_{WZ}=C_{WW}S$, and thus $\| C_{WW}(C_{WW}^2+\delta_n I)^{-1}C_{WZ} -S\| = \| C_{WW}(C_{WW}^2+\delta_n I)^{-1}C_{WW}S -S\|\leq
\delta_n\|(C_{WW}^2+\delta_n I)^{-1}C_{WW}^{\nu}\|\,\|C_{WW}^{-\nu} S\|$.  The eigendecomposition of $C_{WW}$ together with the inequality
$\frac{\delta_n \lambda^\nu}{\lambda^2 + \delta_n}\leq \delta_n^{\min\{1,\nu/2\}}$ ($\lambda\geq 0$) completes the proof.
\end{proof}

\subsubsection*{Acknowledgements}
We thank Arnaud Doucet, Lorenzo Rosasco, Yee Whye Teh and Shuhei Mano for their helpful comments.

\bibliographystyle{plainnat}
\bibliography{fukubib}

\begin{thebibliography}{41}
\providecommand{\natexlab}[1]{#1}
\providecommand{\url}[1]{\texttt{#1}}
\expandafter\ifx\csname urlstyle\endcsname\relax
  \providecommand{\doi}[1]{doi: #1}\else
  \providecommand{\doi}{doi: \begingroup \urlstyle{rm}\Url}\fi

\bibitem[Aronszajn(1950)]{Aronszajn50}
N.~Aronszajn.
\newblock Theory of reproducing kernels.
\newblock \emph{Transactions of the American Mathematical Society}, 68\penalty0
  (3):\penalty0 337--404, 1950.

\bibitem[Baker(1973)]{Baker73}
C.R. Baker.
\newblock Joint measures and cross-covariance operators.
\newblock \emph{Transactions of the American Mathematical Society},
  186:\penalty0 273--289, 1973.

\bibitem[Berlinet and Thomas-Agnan(2004)]{Berlinet_RKHS}
A.~Berlinet and C.~Thomas-Agnan.
\newblock \emph{Reproducing kernel {H}ilbert spaces in probability and
  statistics}.
\newblock Kluwer Academic Publisher, 2004.

\bibitem[Blei and Jordan(2006)]{BleiJordan2006}
D.~Blei and M.~Jordan.
\newblock Variational inference for dirichlet process mixtures.
\newblock \emph{Journal of Bayesian Analysis}, 1\penalty0 (1):\penalty0
  121--144, 2006.

\bibitem[Bowman(1984)]{Bowman1984}
Aedian~W. Bowman.
\newblock {An alternative method of cross-validation for the smoothing of
  density estimates}.
\newblock \emph{Biometrika}, 71\penalty0 (2):\penalty0 353--360, 1984.

\bibitem[Caponnetto and De~Vito(2007)]{CaponnettoDeVito2007}
A.~Caponnetto and E.~De~Vito.
\newblock Optimal rates for regularized least-squares algorithm.
\newblock \emph{Foundations of Computational Mathematics}, 7\penalty0
  (3):\penalty0 331--368, 2007.

\bibitem[Doucet et~al.(2001)Doucet, Freitas, and Gordon]{Docuet_etal_SMC}
A.~Doucet, N.~De Freitas, and N.J. Gordon.
\newblock \emph{Sequential Monte Carlo Methods in Practice}.
\newblock Springer, 2001.

\bibitem[Engl et~al.(2000)Engl, Hanke, and Neubauer]{EnglHankeNeubauer}
H.W. Engl, M.~Hanke, and A.~Neubauer.
\newblock \emph{Regularization of Inverse Problems}.
\newblock Kluwer Academic Publishers, 2000.

\bibitem[Fine and Scheinberg(2001)]{FineScheinberg2001}
S.~Fine and K.~Scheinberg.
\newblock Efficient {SVM} training using low-rank kernel representations.
\newblock \emph{Journal of Machine Learning Research}, 2:\penalty0 243--264,
  2001.

\bibitem[Fukumizu et~al.(2004)Fukumizu, Bach, and Jordan]{Fukumizu04_jmlr}
K.~Fukumizu, F.R. Bach, and M.I. Jordan.
\newblock Dimensionality reduction for supervised learning with reproducing
  kernel {H}ilbert spaces.
\newblock \emph{Journal of Machine Learning Research}, 5:\penalty0 73--99,
  2004.

\bibitem[Fukumizu et~al.(2009{\natexlab{a}})Fukumizu, Bach, and
  Jordan]{Fukumizu_etal09_KDR}
K.~Fukumizu, F.R. Bach, and M.I. Jordan.
\newblock Kernel dimension reduction in regression.
\newblock \emph{Annals of Statistics}, 37\penalty0 (4):\penalty0 1871--1905,
  2009{\natexlab{a}}.
\newblock ISSN 0090-5364.
\newblock \doi{10.1214/08-AOS637}.

\bibitem[Fukumizu et~al.(2009{\natexlab{b}})Fukumizu, Sriperumbudur, Gretton,
  and Schoelkopf]{FukSriGreSch09}
K.~Fukumizu, B.~Sriperumbudur, A.~Gretton, and B.~Schoelkopf.
\newblock Characteristic kernels on groups and semigroups.
\newblock In \emph{Advances in Neural Information Processing Systems 21}, pages
  473--480, Red Hook, NY, 2009{\natexlab{b}}. Curran Associates Inc.

\bibitem[Fukumizu et~al.(2008)Fukumizu, Gretton, Sun, and
  Sch{\"o}lkopf]{Fukumizu_etal_nips07}
Kenji Fukumizu, Arthur Gretton, Xiaohai Sun, and Bernhard Sch{\"o}lkopf.
\newblock Kernel measures of conditional dependence.
\newblock In \emph{Advances in Neural Information Processing Systems 20}, pages
  489--496. MIT Press, 2008.

\bibitem[Gretton et~al.(2007)Gretton, Borgwardt, Rasch, Sch\"{o}lkopf, and
  Smola]{Gretton_etal07}
A.~Gretton, K.M. Borgwardt, M.~Rasch, B.~Sch\"{o}lkopf, and A.~Smola.
\newblock A kernel method for the two-sample-problem.
\newblock In B.~Sch\"{o}lkopf, J.~Platt, and T.~Hoffman, editors,
  \emph{Advances in Neural Information Processing Systems 19}, pages 513--520.
  MIT Press, Cambridge, MA, 2007.

\bibitem[Gretton et~al.(2009{\natexlab{a}})Gretton, Fukumizu, Harchaoui, and
  Sriperumbudur]{Gretton_etal_nips2009}
A.~Gretton, K.~Fukumizu, Z.~Harchaoui, and B.~Sriperumbudur.
\newblock A fast, consistent kernel two-sample test.
\newblock In Y.~Bengio, D.~Schuurmans, J.~Lafferty, C.~K.~I. Williams, and
  A.~Culotta, editors, \emph{Advances in Neural Information Processing Systems
  22}, pages 673--681. 2009{\natexlab{a}}.

\bibitem[Gretton et~al.(2008)Gretton, Fukumizu, Teo, Song, Sch{\"o}lkopf, and
  Smola]{Gretton_etal_nips07}
Arthur Gretton, Kenji Fukumizu, Choon~Hui Teo, Le~Song, Bernhard Sch{\"o}lkopf,
  and Alex Smola.
\newblock A kernel statistical test of independence.
\newblock In \emph{Advances in Neural Information Processing Systems 20}, pages
  585--592. MIT Press, 2008.

\bibitem[Gretton et~al.(2009{\natexlab{b}})Gretton, Fukumizu, and
  Sriperumbudur]{Gretton_etal_AOAS2009}
Arthur Gretton, Kenji Fukumizu, and Bharath~K. Sriperumbudur.
\newblock Discussion of: Brownian distance covariance.
\newblock \emph{Annals of Applied Statistics}, 3\penalty0 (4):\penalty0
  1285--1294, 2009{\natexlab{b}}.

\bibitem[Hofmann et~al.(2008)Hofmann, Sch{\"o}lkopf, and
  Smola]{Hofmann_etal_2008_kernel_AS}
Thomas Hofmann, Bernhard Sch{\"o}lkopf, and Alexander~J. Smola.
\newblock Kernel methods in machine learning.
\newblock \emph{The Annals of Statistics}, 36\penalty0 (3):\penalty0
  1171--1220, 2008.

\bibitem[Julier and Uhlmann(1997)]{JulierUhlmann1997}
S.J. Julier and J.K. Uhlmann.
\newblock A new extension of the kalman filter to nonlinear systems.
\newblock In \emph{Proceedings of AeroSense: The 11th International Symposium
  on Aerospace/Defence Sensing, Simulation and Controls}, 1997.

\bibitem[Kankainen and Ushakov(1998)]{KankainenUshakov1998}
A.~Kankainen and N.G. Ushakov.
\newblock A consistent modification of a test for independence based on the
  empirical characteristic function.
\newblock \emph{Journal of Mathematical Sciencies}, 89:\penalty0 1582--1589,
  1998.

\bibitem[MacEachern(1994)]{MacEachern1994}
S~MacEachern.
\newblock Estimating normal means with a conjugate style dirichlet process
  prior.
\newblock \emph{Communications in Statistics -- Simulation and Computation},
  23\penalty0 (3):\penalty0 727--741, 1994.

\bibitem[MacEachern et~al.(1999)MacEachern, Clyde, and
  Liu]{MacEachern_etal1999}
Steven~N. MacEachern, Merlise Clyde, and Jun~S. Liu.
\newblock Sequential importance sampling for nonparametric bayes models: The
  next generation.
\newblock \emph{The Canadian Journal of Statistics}, 27\penalty0 (2):\penalty0
  251--267, 1999.

\bibitem[Marjoram et~al.(2003)Marjoram, Molitor, Plagnol, and
  Tavare]{Marjoram_etal_2003PNAS}
Paul Marjoram, John Molitor, Vincent Plagnol, and Simon Tavare.
\newblock Markov chain monte carlo without likelihoods.
\newblock \emph{Proceedings of the National Academy of Sciences}, 100\penalty0
  (26):\penalty0 15324--15328, 2003.

\bibitem[Mika et~al.(1999)Mika, Sch{\"o}lkopf, Smola, M{\"u}ller, Scholz, and
  R{\"a}tsch]{Mika99kernelpca}
Sebastian Mika, Bernhard Sch{\"o}lkopf, Alex Smola, Klaus-Robert M{\"u}ller,
  Matthias Scholz, and Gunnar R{\"a}tsch.
\newblock Kernel {PCA} and de-noising in feature spaces.
\newblock In \emph{Advances in Neural Information Pecessing Systems 11}, pages
  536--542. MIT Press, 1999.

\bibitem[Monbet et~al.(2008)Monbet, Ailliot, and Marteau]{Monbet_etal2008}
V.~Monbet, P.~Ailliot, and P.F. Marteau.
\newblock $l^1$-convergence of smoothing densities in non-parametric state
  space models.
\newblock \emph{Statistical Inference for Stochastic Processes}, 11:\penalty0
  311--325, 2008.

\bibitem[M{\"u}ller and Quintana(2004)]{MullerQuintana2004}
P.~M{\"u}ller and F.A. Quintana.
\newblock Nonparametric bayesian data analysis.
\newblock \emph{Statistical Science}, 19\penalty0 (1):\penalty0 95--110, 2004.

\bibitem[Rudemo(1982)]{Rudemo1982}
Mats Rudemo.
\newblock Empirical choice of histograms and kernel density estimators.
\newblock \emph{Scandinavian Journal of Statistics}, 9\penalty0 (2):\penalty0
  pp. 65--78, 1982.

\bibitem[Sch{\"o}lkopf and Smola(2002)]{SchoelkopfSmola_book}
B.~Sch{\"o}lkopf and A.J. Smola.
\newblock \emph{Learning with Kernels}.
\newblock MIT Press, 2002.

\bibitem[Sisson et~al.(2007)Sisson, Fan, and Tanaka]{Sisson_etal2007}
S.~A. Sisson, Y.~Fan, and Mark~M. Tanaka.
\newblock Sequential monte carlo without likelihoods.
\newblock \emph{Proceedings of the National Academy of Sciences}, 104\penalty0
  (6):\penalty0 1760--1765, 2007.

\bibitem[Smale and Zhou(2007)]{SmaleZhou2005}
Steve Smale and Ding-Xuan Zhou.
\newblock Learning theory estimates via integral operators and their
  approximation.
\newblock \emph{Constructive Approximation}, 26:\penalty0 153--172, 2007.

\bibitem[Song et~al.(2009)Song, Huang, Smola, and Fukumizu]{Song_etal_ICML2009}
L.~Song, J.~Huang, A.~Smola, and K.~Fukumizu.
\newblock Hilbert space embeddings of conditional distributions with
  applications to dynamical systems.
\newblock In \emph{Proceedings of the 26th International Conference on Machine
  Learning (ICML2009)}, pages 961--968. 2009.

\bibitem[Song et~al.(2010{\natexlab{a}})Song, Gretton., and
  Guestrin]{Song_etal_AISTATS2010}
L.~Song, A.~Gretton., and C.~Guestrin.
\newblock Nonparametric tree graphical models via kernel embeddings.
\newblock In \emph{Proceedings of AISTATS 2010}, pages 765--772,
  2010{\natexlab{a}}.

\bibitem[Song et~al.(2010{\natexlab{b}})Song, Siddiqi, Gordon, and
  Smola]{Song_etal_ICML2010}
L.~Song, S.~M. Siddiqi, G.~Gordon, and A.~Smola.
\newblock Hilbert space embeddings of hidden markov models.
\newblock In \emph{Proceedings of the 27th International Conference on Machine
  Learning (ICML2010)}, pages 991--998. 2010{\natexlab{b}}.

\bibitem[Song et~al.(2011)Song, Gretton, Bickson, Low, and
  Guestrin]{SonGreBicLowGue11}
L.~Song, A.~Gretton, D.~Bickson, Y.~Low, and C.~Guestrin.
\newblock Kernel belief propagation.
\newblock In \emph{Proceedings of AISTATS 2011}, pages 707--715, 2011.

\bibitem[Sriperumbudur et~al.(2010)Sriperumbudur, Gretton, Fukumizu,
  Sch{\"o}lkopf, and Lanckriet]{Sriperumbudur_etal2010JMLR}
Bharath~K. Sriperumbudur, Arthur Gretton, Kenji Fukumizu, Bernhard
  Sch{\"o}lkopf, and Gert~R.G. Lanckriet.
\newblock Hilbert space embeddings and metrics on probability measures.
\newblock \emph{Journal of Machine Learning Research}, 11:\penalty0 1517--1561,
  2010.

\bibitem[Sriperumbudur et~al.(2011)Sriperumbudur, Fukumizu, and
  Lanckriet]{Sriperumbudur_etal2011JMLR}
Bharath~K. Sriperumbudur, Kenji Fukumizu, and Gert Lanckriet.
\newblock Universality, characteristic kernels and rkhs embedding of measures.
\newblock \emph{Journal of Machine Learning Research}, 12:\penalty0 2389--2410,
  2011.

\bibitem[Steinwart(2001)]{Steinwart01}
I.~Steinwart.
\newblock On the influence of the kernel on the consistency of support vector
  machines.
\newblock \emph{Journal of Machine Learning Research}, 2:\penalty0 67--93,
  2001.

\bibitem[Steinwart and Christmann(2008)]{SteChr08}
Ingo Steinwart and Andreas Christmann.
\newblock \emph{Support Vector Machines}.
\newblock Information Science and Statistics. Springer, 2008.

\bibitem[Tavar{\'e} et~al.(1997)Tavar{\'e}, Balding, Griffithis, and
  Donnelly]{Tavare_etal_1997ABC}
S.~Tavar{\'e}, D.J. Balding, R.C. Griffithis, and P.~Donnelly.
\newblock Inferring coalescence times from dna sequece data.
\newblock \emph{Genetics}, 145:\penalty0 505--518, 1997.

\bibitem[Thrun et~al.(1999)Thrun, Langford, and Fox]{Thurn_etal_ICML1999}
S.~Thrun, J.~Langford, and D.~Fox.
\newblock Monte carlo hidden markov models: Learning non-parametric models of
  partially observable stochastic processes.
\newblock In \emph{Proceedings of International Conference on Machine Learning
  (ICML 1999)}, pages 415--424, 1999.

\bibitem[West et~al.(1994)West, M{\"u}ller, and Escobar]{West_etal1994}
Mike West, Peter M{\"u}ller, and Michael~D. Escobar.
\newblock Hierarchical priors and mixture models, with applications in
  regression and density estimation.
\newblock In P.~Freeman et~al, editor, \emph{Aspects of Uncertainty: A Tribute
  to D.V. Lindley}, pages 363--386. 1994.

\end{thebibliography}

\end{document}